\documentclass{alt2027}
\usepackage{times}
\usepackage{xurl}
\usepackage{microtype}
\usepackage{mathtools}
\usepackage{booktabs}
\usepackage{enumitem}
\usepackage{longtable}
\usepackage{array}
\usepackage{ragged2e}
\usepackage{pdflscape}
\usepackage{makecell}

\usepackage{latexsym}
\usepackage{verbatim}
\usepackage{bbm}

{\begin{list}               
    {$\bullet$ \hfill}{
        \setlength{\leftmargin}{\parindent}
        \setlength{\parsep}{0.04\baselineskip}
        \setlength{\itemsep}{0.5\parsep}
        \setlength{\labelwidth}{\leftmargin}
        \setlength{\labelsep}{0em}}
    }
{\end{list}}

\providecommand{\cref}[1]{Chapter~\ref{chap:#1}}

\providecommand{\R}{\ensuremath{\mathbb{R}}}

\renewcommand{\vec}[1]{\ensuremath{\boldsymbol{#1}}}

\newcommand{\E}{\mathbb{E}\,}

\DeclareMathOperator{\tr}{tr}
\DeclareMathOperator{\diag}{diag}

\providecommand{\vm}{\vec{m}}

\allowdisplaybreaks
\hypersetup{pdfauthor={}}

\newcommand{\range}{\operatorname{range}}
\DeclareMathOperator{\rank}{rank}

\title[Finite-Sample Trajectories of Regularized Lin-RFM]
{Exact Dynamics and Finite-Sample Trajectory Recovery\\
 of Linear Recursive Feature Machines}

\altauthor{
Andrew Cheng\textsuperscript{1},
Bobak T. Kiani\textsuperscript{2},
Yue M. Lu\textsuperscript{1},
Adityanarayanan Radhakrishnan\textsuperscript{3}
\\[0.5em]
\textsuperscript{1}Harvard University \qquad
\textsuperscript{2}Bowdoin College \qquad
\textsuperscript{3}Massachusetts Institute of Technology
}

\begin{document}
\maketitle
\begin{abstract}

Recursive feature machines (RFMs) learn representations of data by alternating between fitting a predictor to a dataset and updating features of that predictor using the average gradient outer product (AGOP). Connections between AGOPs and feature learning in neural networks motivate linear RFMs as a simple setting for analyzing how representations evolve during training. 
Here, we study the dynamics and statistics of linear RFM in noisy multi-output regression with isotropic sub-Gaussian input data and targets generated by a low-rank teacher matrix of dimension $d$. 
We extend the known connection between linear RFM and iteratively
reweighted least squares from the interpolating setting to 
ridge-regularized multi-output regression with noise.
We show that the learned feature matrix remains close to its infinite-data ideal counterpart at every iteration. Namely, for $n$ samples, we show the error in the feature matrix decays as $O(\sqrt{d/n})$ with high probability.
Experiments on real-world text and single-cell gene-expression data
illustrate the features learned by this simple linear model.
\end{abstract}
\section{Introduction}
\label{sec:introduction}

Feature learning seeks representations that expose the structure relevant to a prediction task. In neural networks, this takes the form of adapting intermediate representations during training, raising the question of which features are learned and how they emerge.~\cite{radhakrishnan2024mechanism} proposed a connection between the directions emphasized by trained neural networks and their average gradient outer products (AGOPs)~\citep{hardle1989investigating,hristache2001structure,trivedi2020expected}. The AGOP measures a predictor's sensitivity to input directions marking directions most relevant to a given prediction task.

This line of work motivated the development of a training algorithm that explicitly uses AGOPs called recursive feature machines (RFMs). RFMs alternate between fitting a predictor and updating the representation using its AGOP~\citep{radhakrishnan2025linear}. Recent work has found remarkable applications for RFM including the monitoring and steering of AI models~\citep{beaglehole2026universal, hsu2026contextual, davarmanesh2026efficient}, state-of-the art tabular methods~\citep{beaglehole2026xrfm, gedon2024uncertainty}, feature learning of neural networks~\citep{boixadsera2026fact, mallinar2025emergence}, and preconditioning for 
optimization~\citep{depavia2026faster}.

Studying RFM in the linear setting provides a tractable setting to better understand the mechanism behind feature learning and its implications. \cite{radhakrishnan2025linear} introduces linear RFM (lin-RFM) and connects it to low-rank matrix recovery through
iteratively reweighted least squares (IRLS)~\citep{daubechies2010irls, mohan2012iterative, lai2013improved, nikolova2005analysis} in the noiseless and interpolating setting where the data can be perfectly fit.

Here we study regularized
lin-RFM in noisy multi-output regression. Each iteration fits a ridge
predictor and updates the representation using a spectral
transform of its AGOP. Although the end-to-end predictor remains linear, the evolving representation changes the penalty or weighting assigned to different input directions.
This iterative procedure generates a trajectory of
feature-reweighting matrices whose evolution depends on
the data, noise level, and regularization parameters.
Because every iteration reuses the same finite dataset,
estimation errors may be amplified by subsequent updates.
It is therefore unclear whether these errors remain
controlled throughout the trajectory.

Our focus here is to understand rigorously how representations evolve in the linear RFM algorithm.
We address three questions:
\begin{enumerate}[leftmargin=2em]
\item \textbf{Exact dynamics, finite-time behavior, and rates}: \emph{What are the exact dynamics of the learned representation and predictor at each finite time step?}
\item \textbf{Finite-sample trajectory recovery}: \emph{How much data are required for the empirical representation and predictor to follow the exact trajectory?}
\item \textbf{Convergence results:} \emph{Does the RFM algorithm converge, and if so, can we characterize its limiting point?}
\end{enumerate}

\paragraph{Main contributions.}
We answer the three questions above. We study a student-teacher setting where the linear RFM algorithm is trained to predict noisy data $y=Ex+\xi$ generated by a low rank teacher matrix $E$ and noise $\xi$. 
\begin{enumerate}[leftmargin=2em]
\item \textbf{Exact dynamics, finite-time behavior, and rates.}
Proposition~\ref{prop:population-dynamics} characterizes the exact learned
representation and predictor at every finite iteration when trained on isotropic data. The first update identifies the task-dependent subspace exactly. Later updates preserve this subspace and only reweigh the individual directions according to an explicit
one-dimensional spectral recursion. We show this spectral recursion is of the type analyzed in the framework of half-quadratic minimization methods~\citep{nikolova2005analysis}. Their theory studies convergence rates and establishes the convergence profile of our exact dynamics which shows linear convergence (after a finite-time burn-in period) to the optimal point. 

\item \textbf{Finite-sample trajectory recovery.}
Theorem~\ref{thm:deterministic-tracking} is our main
technical result. Together with
Corollary~\ref{cor:subgaussian-tracking}, it shows that,
under isotropic sub-Gaussian input data and bounded
teacher rank, the entire empirical representation and
predictor paths remain uniformly close to their
population counterparts with high probability. For fixed dimension and confidence level, the
discrepancy decreases at the usual $n^{-1/2}$ rate in the sample size. Crucially, this error does not
accumulate across iterations, even though every update reuses the same data. The exact population dynamics therefore remain informative
for the finite-data algorithm at every stage of its trajectory.

\item \textbf{Convergence path and implicit bias.}
We show that regularized lin-RFM is a specialized instance of the IRucLq-M algorithm~\citep{lai2013improved} which covers a class of regularized multivariate regression problems. Lin-RFM’s updates and its corresponding convergence  map to a specific choice of the regularization in IRucLq-M. This choice is also the regularized analogue of the lin-RFM optimization~\citep{radhakrishnan2025linear} who analyzed the interpolation setting where the training data can be perfectly fit.
Our exact dynamics analysis characterizes the path to this limiting point and our
sample complexity analysis transfers this behavior to the empirical path. Namely, the empirical path follows the population trajectory uniformly in time up to a statistical neighborhood whose radius decreases with sample size and, after a finite number of steps, inherits the classical linear convergence profile of IRLS up to this statistical neighborhood~\citep{daubechies2010irls}.

\end{enumerate}

We do not claim a novel algorithm, rather we specialize the setting of unconstrained matrix-IRLS (e.g.,\citep{lai2013improved}) to multivariate output regression and leverage this additional structure to establish our results. To the best of our knowledge, prior IRLS analyses do not combine an exact finite-time trajectory, uniform statistical control over the empirical path, and sample complexity recovery rates.

\paragraph{Relation to adjacent theory.}
The connection between RFMs and feature learning in neural networks was established by~\citep{radhakrishnan2024mechanism}. Subsequent work has developed both sides of this connection: \citep{boixadsera2026fact} derives a first-principles characterization of neural-network feature geometry at convergence and relates it to AGOP, while \citep{zhu2025reweighted} shows that iterative gradient-based reweighting enables kernel methods to efficiently learn sparse hierarchical functions. Gradient outer-product matrices also have a longer history in dimension reduction and active-subspace methods~\citep{constantine2015active}. More recently, related gradient-outer-product geometry has been used for optimization: \citep{depavia2026faster} shows that reparameterizing adaptive methods in the eigenbasis of an expected gradient outer-product matrix can improve convergence under spectral decay. This connects to a broader literature on curvature-aware and matrix-preconditioned optimization for neural networks~\citep{vyas2025soap,gupta2018shampoo,martens2015kfac}.

Our work is also related to iteratively reweighted least-squares (IRLS) methods. Sparse-vector IRLS was analyzed by Daubechies, DeVore, Fornasier, and G\"unt\"urk~\cite{daubechies2010irls}, while subsequent work developed IRLS methods for sparse recovery and low-rank matrix problems~\citep{mohan2012iterative,lai2013improved}. \citep{radhakrishnan2025linear} subsequently connected lin-RFM to IRLS in the interpolating setting and interpreted the evolving weight as a learned feature transformation. Recent high-dimensional trajectory theory for related reweighted algorithms uses fresh independent batches across iterations~\citep{kaushik2024precise}, while statistical guarantees for low-rank regression characterize the accuracy of regularized estimators~\citep{negahban2011estimation,bunea2011optimal}. In contrast, our analysis controls each empirical iterate relative to its population counterpart uniformly over all iterations while reusing a single fixed dataset.

\paragraph{Organization.}
Section~\ref{sec:algorithm} introduces the regularized lin-RFM recursion, its streaming implementation, and its fixed-data relationship to matrix IRLS. Section~\ref{sec:implicit_bias} then gives the complementary optimization perspective and characterizes the fixed-data limiting point. Section~\ref{sec:population-rate} derives the exact finite-time population trajectory and its eventual linear convergence rate. Section~\ref{sec:sample-complexity} proves uniform non-asymptotic finite-sample trajectory recovery under isotropic sub-Gaussian data. Complete proofs and additional experiments are deferred to the supplement.

\section{Regularized lin-RFM}\label{sec:algorithm}
We consider supervised multivariate regression from training pairs
$(x_i,y_i)_{i=1}^n$, where $x_i\in\mathbb R^d$ and
$y_i\in\mathbb R^p$ and train a predictor of the form
\[
f(x)=ABx=:Fx.
\]
Here $B\in\mathbb S_{++}^d$ is the learned representation,
$A\in\mathbb R^{p\times d}$ is the linear output layer, and
$F=AB$ is the resulting end-to-end predictor. Although $F$
remains linear in the original input, this factorization separates
feature learning from prediction. Here, the matrix $B$ determines the linear
features $Bx$ presented to the final linear layer. Assuming $B$ is invertible, the factorization does
not enlarge the class of end-to-end linear predictors. Rather, it
changes the representation in order to emphasize task-sensitive feature directions and down-weight irrelevant directions.

Let $X\in\R^{d\times n}$ contain the input vectors as columns and let $Y\in\R^{p\times n}$ contain the corresponding responses. Define the empirical covariance and input--output cross-moment
\[
S_n:=\frac1nXX^\top,
\qquad  E_n:=\frac1nYX^\top.
\]

Starting from the identity representation, regularized lin-RFM
alternates between fitting the linear output layer on the current representation and
using the resulting predictor to update the representation:

\begingroup
\setlength{\algomargin}{1.2em}
\begin{algorithm2e}[H]
\caption{Regularized lin-RFM}
\label{alg:regularized-lin-rfm}\label{eq:feature-learning-recursion}
\DontPrintSemicolon
\LinesNumbered
\KwIn{Training matrices $X\in\R^{d\times n}$ and
$Y\in\R^{p\times n}$; parameters $\lambda,\epsilon>0$;
$T\in\mathbb{N}$.}
$B_0\gets I_d$\;
\For{$t=0,\ldots,T$}{
    $\displaystyle A_t\gets
    \underset{A\in\R^{p\times d}}{\arg\min}
    \left\{\frac{1}{2n}\|AB_tX-Y\|_F^2
    +\frac{\lambda}{2}\|A\|_F^2\right\}$\;
    \tcp{fit the ridge readout}
    $F_t\gets A_tB_t$\;
    \If{$t<T$}{
        $B_{t+1}\gets
        \bigl(F_t^\top F_t+\epsilon I_d\bigr)^{1/4}$\;
        \tcp{update the features}
    }
}
\end{algorithm2e}
\endgroup

\paragraph{Streaming implementation.}
The complete trajectory of Algorithm~\ref{alg:regularized-lin-rfm} depends on the data only through the sufficient statistics $(S_n, E_n)$. These matrices can be accumulated in one pass, so the full sample matrices need not be stored. If each input has at most $k\ll d$ nonzero entries and the response is one-hot, then one sample updates $S_n$ in $O(k^2)$ operations and $ E_n$ in $O(k)$ operations. Once the moments have been formed, every RFM update is independent of $n$; the state occupies $O(d^2+pd)$ memory. The method therefore removes the sample-size bottleneck in the $n\gg d$ regime, although the remaining dense linear algebra has runtime dependent on the admissible feature dimension.

\paragraph{Regularized lin-RFM as matrix IRLS.}
Iteratively reweighted least squares (IRLS) alternates between solving a
weighted least-squares problem and updating the weights from the current
estimate. The method has a long history in robust
regression~\citep{holland1977robust} and was later analyzed for sparse-vector
recovery~\citep{daubechies2010irls}. Matrix IRLS applies the same reweighting
principle to low-rank estimation: directions associated with smaller singular
values of the current estimate receive stronger quadratic penalties in the
next fit. Constrained matrix-recovery formulations were analyzed by
\citet{mohan2012iterative} and \citet{fornasier2011lowrank}, with each iterate
required to fit the observations exactly. \citet{lai2013improved} instead
study penalized formulations, where a quadratic residual term allows
imperfect data fit. Their fixed-smoothing, Schatten-one matrix update is the
one we specialize to multivariate regression below.

At the level of the update map, regularized lin-RFM is exactly a
reparameterization of the fixed-smoothing, Schatten-one
matrix IRLS method of~\cite{lai2013improved}, specialized to multivariate regression. 
To make this explicit, one can write Algorithm~\ref{alg:regularized-lin-rfm} in terms of the evolution of the predictor $F_t$:

\begin{equation}
\boxed{
\begin{aligned}
&\textbf{Initialize:}
&&B_0=I_d.\\[1mm]
&\text{For } t=0,1,\ldots. \\
&\textbf{Predictor fit:}
&&F_t= E_n B_t\bigl(B_tS_nB_t+\lambda I_d\bigr)^{-1}B_t,\\
&\textbf{Feature update:}
&&B_{t+1}=\bigl(F_t^\top F_t+\epsilon I_d\bigr)^{1/4}.
\end{aligned}}
\label{eq:main-recursion}
\end{equation}
 In the unconstrained 
setting studied by Lai, Xu, and Yin, the observations need not be fitted
exactly; instead, the residual enters through a quadratic data-fit term. Their
matrix objective is
\begin{equation}
\mathcal L_q(F;\varepsilon,\lambda)
=
\operatorname{tr}\!\left(
(F^\top F+\varepsilon^2 I_d)^{q/2}
\right)
+
\frac{1}{2\lambda}
\|\mathcal A(F)-\mathbf b\|_2^2.
\label{eq:lxy_objective}
\end{equation}
To obtain multivariate regression, set $q=1$, identify
$\varepsilon^2=\epsilon$, and take
\[
\mathcal A_n(F)
:=
\frac{1}{\sqrt n}\operatorname{vec}(FX),
\qquad
\mathbf b_n
:=
\frac{1}{\sqrt n}\operatorname{vec}(Y).
\]
Since
\[
\|\mathcal A_n(F)-\mathbf b_n\|_2^2
=
\frac1n\|FX-Y\|_F^2,
\]
the specialized objective is, up to multiplication by a positive constant,
\begin{equation}
\mathcal J_{n,\epsilon}(F)
=
\frac{1}{2n}\|FX-Y\|_F^2
+
\lambda\operatorname{tr}\!\left(
(F^\top F+\epsilon I_d)^{1/2}
\right).
\label{eq:smoothed-trace-objective}
\end{equation}
The second term is a smoothed trace-norm penalty and therefore favors
low-rank predictors. At the current predictor $F_t$, the corresponding IRLS step forms the weight
\[
P_t
:=
(F_t^\top F_t+\epsilon I_d)^{-1/2}
\]
and solves
\begin{equation}
F_{t+1}
=
\underset{F\in\mathbb R^{p\times d}}{\arg\min}
\left\{
\frac1{2n}\|FX-Y\|_F^2
+
\frac{\lambda}{2}\operatorname{tr}(P_tF^\top F)
\right\}
=
E_n(S_n+\lambda P_t)^{-1}.
\label{eq:weighted-regression-irls}
\end{equation}

The connection to regularized lin-RFM follows from its feature update:
\[
B_{t+1}
=
(F_t^\top F_t+\epsilon I_d)^{1/4},
\qquad\text{so that}\qquad
P_t=B_{t+1}^{-2}.
\]
Substituting this identity into the IRLS predictor step gives
\[
F_{t+1}
=
E_n
\bigl(S_n+\lambda B_{t+1}^{-2}\bigr)^{-1},
\]
which is precisely the next predictor update in
\eqref{eq:main-recursion}. Thus the two formulations use the same matrix in complementary ways:
\[
\underbrace{P_t}_{\text{IRLS penalty weight}}
=
\underbrace{B_{t+1}^{-2}}_{\text{inverse learned feature weighting}}.
\]
In the IRLS interpretation, $P_t$ adaptively reweights the quadratic
regularizer. In the RFM interpretation, $B_{t+1}$ is the learned feature map,
and $B_{t+1}^2$ records how strongly the algorithm weights different feature
directions. Directions assigned larger feature weights are penalized less in
the next regression step. 

This correspondence identifies the objective and its limiting
point, but it does not describe the trajectory at
finite time. Our contribution is to characterize that trajectory exactly in
population, establish its non-asymptotic convergence rates, and prove that the empirical
representation and predictor track it uniformly even when the same sample is
reused at every iteration. We return to the fixed-data endpoint in
Section~\ref{sec:implicit_bias}.

\paragraph{Recovering lin-RFM and constrained matrix IRLS.}
Algorithm~\ref{alg:regularized-lin-rfm} is the unconstrained analogue of the constrained IRLS-1 and lin-RFM
correspondence established by~\citet{radhakrishnan2025linear} specialized to the regression setting. Furthermore, the limiting point of the algorithm recovers known implicit bias results from the analysis of linear neural networks and matrix factorization~\citep{gunasekar2017implicit,gunasekar2018implicit,lawrence2021implicit,arora2019implicit,bai2024connectivity}. This is made precise in Remark~\ref{remark:implicit-bias}.

\section{Optimization endpoint and implicit bias}
\label{sec:implicit_bias}

Section~\ref{sec:algorithm} showed that, after matrix-IRLS framework of Lai, Xu, and Yin~\citeyearpar{lai2013improved} to multivariate regression, its
predictor update coincides with regularized lin-RFM under the identification $P_t=B_{t+1}^{-2}.$ This correspondence tells us where the recursion converges on any fixed
dataset, independently of the teacher model or the distribution of the data.

\begin{proposition}[Fixed-data endpoint]
\label{prop:implicit-bias}
Fix $X\in\R^{d\times n}$ and $Y\in\R^{p\times n}$, let
$\lambda,\epsilon>0$, and let $(B_t,F_t)$ follow
\eqref{eq:main-recursion}. Then
\[
F_t\longrightarrow F_\star,
\]
where $F_\star$ is the unique minimizer of the smoothed trace-norm objective
\begin{equation}
F_\star
:=
\underset{F\in\R^{p\times d}}{\arg\min}
\left\{
\frac1{2n}\|FX-Y\|_F^2
+
\lambda\operatorname{tr}\!\left(
(F^\top F+\epsilon I_d)^{1/2}
\right)
\right\}.
\label{eq:fixed-data-objective}
\end{equation}
Moreover,
\[
B_t\longrightarrow
B_\star
:=
\bigl(F_\star^\top F_\star+\epsilon I_d\bigr)^{1/4}.
\]
\end{proposition}

\begin{proof}
By the correspondence in Section~\ref{sec:algorithm}, the
convergence argument of
\citep[Lemma~3.4 and Theorem~3.1]{lai2013improved} shows that $(F_t)$
is bounded and every accumulation point is stationary for
the objective in~\eqref{eq:fixed-data-objective}.
This objective is strictly convex so every accumulation point equals its unique minimizer
$F_\star$. Hence $F_t\to F_\star$. The convergence of $B_t$ follows from
$B_{t+1}=(F_t^\top F_t+\epsilon I_d)^{1/4}$
and continuity of the positive matrix fourth root.
\end{proof}

This identifies both the limiting predictor and the corresponding
learned representation. We do not claim the objective
\eqref{eq:fixed-data-objective}, or convergence to its minimizer, as new. The endpoint identifies the limiting point but does
not describe the finite-time path, how quickly the algorithm approaches the
endpoint, or how sampling error affects the trajectory. Proposition~\ref{prop:population-dynamics} gives the exact population path and
its eventual linear convergence, while
Theorem~\ref{thm:deterministic-tracking} shows that the empirical path remains
uniformly close to it.

\begin{remark}[Connection to IRLS and constrained nuclear norm]\label{remark:implicit-bias}
The fixed point $F_\star$ satisfies
\[
F_\star S_n-\widehat E_n
+\lambda F_\star
\left(F_\star^\top F_\star+\epsilon I\right)^{-1/2}
=0,
\]
which is the first-order condition for the smoothed trace-norm regression problem in~\eqref{eq:fixed-data-objective}. Proposition~\ref{prop:implicit-bias} is the noisy regression analogue of the IRLS-$1$ and lin-RFM analysis studied in \citep{mohan2012iterative,radhakrishnan2025linear}.
Namely, in the noiseless case where the data can be perfectly fit, dividing by $\lambda$ gives
\[
\frac{1}{2n\lambda}\|FX-Y\|_F^2
+
\operatorname{tr} \left((F^\top F+\epsilon I)^{1/2}\right),
\]
and the limit $\lambda\downarrow 0$ yields the constrained smoothed trace-norm problem
\[
\min_F \operatorname{tr} \left((F^\top F+\epsilon I)^{1/2}\right)
\qquad
\text{subject to}
\qquad
FX=Y.
\]
If in addition $\epsilon\downarrow 0$, this reduces further to the constrained nuclear norm minimization
\[
\min_F \|F\|_*
\qquad
\text{subject to}
\qquad
FX=Y.
\]
This implicit bias to the nuclear norm has been recovered in various related settings including analysis of linear neural networks and matrix factorization \citep{gunasekar2017implicit,gunasekar2018implicit,lawrence2021implicit,arora2019implicit,bai2024connectivity}.
\end{remark}


\section{Exact finite-time population dynamics}\label{sec:population-rate}
Fix $\lambda,\epsilon>0$. For our theory, we specialize to
square outputs $p=d$ and the isotropic teacher model
\begin{equation}
y=Ex+\xi,
\qquad
\E[x]=0,
\qquad
\E[xx^\top]=I_d,
\qquad
\E[\xi\mid x]=0.
\label{eq:teacher-model-main}
\end{equation}
We work in the bounded-rank regime
\[
1\le r:=\rank(E)\le r_0,
\]
where $r_0$ is fixed independently of $n$ and $d$.
We also assume $\E\|\xi\|_2^2<\infty$, so that
$\E[yx^\top]=E$.
A bar denotes the population recursion obtained by
replacing $(S_n,E_n)$ with $(I_d,E)$, with initialization $\bar B_0=I_d$.
All matrix norms in this section are operator norms.

Let $E^\top E=V\diag(\mu_1,\ldots,\mu_d)V^\top$, where
$\mu_1\ge\cdots\ge\mu_r>0$ and $\mu_{r+1}=\cdots=\mu_d=0$.
For $\mu\ge0$, define
\begin{equation}
b_0(\mu)=1,
\qquad
b_{t+1}(\mu)
=\left[\epsilon+\mu\left(
\frac{b_t(\mu)^2}{b_t(\mu)^2+\lambda}
\right)^2\right]^{1/4},
\label{eq:scalar-recursion-award}
\end{equation}
and write
$b_t(E^\top E):=V\diag(b_t(\mu_1),\ldots,b_t(\mu_d))V^\top$.
In the population setting, the feature matrix and predictor evolve according to the recursion in $b_t$.

\begin{proposition}[Exact population path]
\label{prop:population-dynamics}
For every $t\ge0$,
\[
\bar B_t=b_t(E^\top E),
\qquad
\bar F_t=E\bar B_t^2(\bar B_t^2+\lambda I_d)^{-1}.
\]
\end{proposition}

\begin{proof}
The population ridge update gives the displayed formula for $\bar F_t$.
Since $\bar B_0=I_d=b_0(E^\top E)$, the representation formula holds initially.
If $\bar B_t=b_t(E^\top E)$, then
\[
\bar F_t^\top\bar F_t
=V\diag\left(
\mu_i\left(
\frac{b_t(\mu_i)^2}{b_t(\mu_i)^2+\lambda}
\right)^2
\right)_{i=1}^dV^\top.
\]
Taking the positive fourth root in
$\bar B_{t+1}=(\bar F_t^\top\bar F_t+\epsilon I_d)^{1/4}$
gives $\bar B_{t+1}=b_{t+1}(E^\top E)$, completing the induction.
\end{proof}

Thus the population representation remains diagonal in the fixed teacher
basis $V$. Each iteration changes only the weight assigned to each direction.

\paragraph{Connection with multiplicative half-quadratic minimization.}
Fix $\mu\ge0$ and define the predictor coordinate
\begin{equation}
f_t(\mu):=\sqrt{\mu}\frac{b_t(\mu)^2}{b_t(\mu)^2+\lambda}.
\label{eq:population-predictor-coordinate}
\end{equation}
Suppressing the argument $\mu$, the scalar recursion becomes
\begin{equation}
b_{t+1}=(f_t^2+\epsilon)^{1/4},
\qquad
f_{t+1}=g_\mu(f_t),
\qquad
g_\mu(f):=
\frac{\sqrt{\mu}\sqrt{f^2+\epsilon}}
{\sqrt{f^2+\epsilon}+\lambda},
\label{eq:population-hq-map}
\end{equation}
with $f_0=\sqrt{\mu}/(1+\lambda)$.

This is the majorization--minimization algorithm, studied by Nikolova and
Ng~\citeyearpar{nikolova2005analysis}, applied to
\[
j_\mu(f):=\frac12(f-\sqrt{\mu})^2+\lambda\sqrt{f^2+\epsilon}.
\]
Indeed, specializing their objective
\citep[equations~(1.2)--(1.3)]{nikolova2005analysis} to
$A=G=1$, $y=\sqrt{\mu}$, $\beta=2\lambda$, and
$\varphi(f)=\sqrt{f^2+\epsilon}$ gives $2j_\mu$.
Their multiplicative half-quadratic formulation
\citep[equation~(2.4)]{nikolova2005analysis}, together with the weight
and minimizer formulas
\citep[equations~(2.5), (2.7)--(2.8)]{nikolova2005analysis}, gives
\[
w_{t+1}=\frac1{\sqrt{f_t^2+\epsilon}}=b_{t+1}^{-2},
\qquad
f_{t+1}=\frac{\sqrt{\mu}}{1+\lambda w_{t+1}}=g_\mu(f_t).
\]
Here $w$ denotes their auxiliary variable, to distinguish it from our
representation weight $b$.

These choices satisfy the hypotheses of their Theorem~2.3. Under this specialization, the convergence theorem
\citep[Theorem~2.3(ii)]{nikolova2005analysis} and local-rate result
\citep[Corollary~2.5]{nikolova2005analysis},
with the fixed-point Jacobian evaluated using
\citep[equations~(2.29), (2.36)]{nikolova2005analysis},
yield the following corollary.

\begin{corollary}[Population convergence and local contraction factor]
\label{cor:population-convergence}
For every $\mu\ge0$, the sequence $b_t(\mu)$ converges to the unique
fixed point $b_\star(\mu)\ge\epsilon^{1/4}$
of~\eqref{eq:scalar-recursion-award}.
Set
\[
\bar B_\star:=b_\star(E^\top E),
\qquad
\bar F_\star:=
E\bar B_\star^2(\bar B_\star^2+\lambda I_d)^{-1}.
\]
The Jacobian of the reduced map on
$(b_t(\mu_1),\ldots,b_t(\mu_r))$ at its fixed point is diagonal,
with spectral radius
\[
\kappa_\star
:=\max_{1\le i\le r}
\frac{\mu_i\lambda}
{\bigl(b_\star(\mu_i)^2+\lambda\bigr)^3}<1.
\]
For every $\rho\in(\kappa_\star,1)$, there exist a finite $T_\rho$
and $C_\rho<\infty$ such that, for all $t\ge T_\rho$,
\begin{equation}
\|\bar B_t-\bar B_\star\|+\|\bar F_t-\bar F_\star\|
\le C_\rho\rho^{\,t-T_\rho}.
\label{eq:population-geometric-award}
\end{equation}
The constants depend only on $\{\mu_i\}_{i=1}^r$ and
$(\lambda,\epsilon,\rho)$.
\end{corollary}


\section{Finite-sample trajectory recovery}\label{sec:sample-complexity}

The population analysis above assumes that the covariance and input--output cross-moment are known exactly. In practice, the algorithm replaces them with the empirical moments
$S_n$ and $ E_n$. Since these same two matrices are reused at every iteration, it is not immediate that a small estimation error will remain small: in principle, the nonlinear updates could repeatedly amplify it.

Our main stability result shows that this does not happen. Define the empirical moment error
\[
\Delta_n
:=
\|S_n-I_d\|
+
\| E_n-E\|.
\]

\begin{theorem}[Uniform stability from empirical moments]
\label{thm:deterministic-tracking}
Fix $\lambda,\epsilon>0$, an integer $r_0\ge1$, and an upper
bound $\|E\|\le M_E<\infty$ on the teacher norm.
There exist constants $c,C>0$, depending only on
$(r_0,M_E,\lambda,\epsilon)$, such that the following holds.
If $\rank(E)\le r_0$ and $\Delta_n\le c$,
then
\begin{equation}
\sup_{t\ge0}
\max\left\{
\|B_t-\bar B_t\|,
\,
\|F_t-\bar F_t\|
\right\}
\le C\Delta_n.
\label{eq:uniform-path-award}
\end{equation}
\end{theorem}

The supremum over $t$ shows that the finite-sample feature learning trajectory
tracks the population without growth of the error bound over the steps of the algorithm.
intermediate representations.

\paragraph{From scalar dynamics to matrix stability.}
The population iterates $\bar{B}_t$ as discussed in Proposition~\ref{prop:population-dynamics} remain diagonal in the eigenbasis of $E^\top E$, changing only the weights assigned to these directions. The empirical moment errors $S_n - I_d$ and $E_n - E$, however, can also alter the feature directions as the empirical iterates need not remain diagonal in this basis. 
Thus, scalar convergence in the population does not immediately control all perturbation encountered by the empirical algorithm. 

To address this, our proof separates two sources of error: (i) estimating the moments and (ii) propagation of the existing trajectory error. The first effect is uniformly bounded by $C\Delta_n$ by standard analytic arguments.
The main challenge is the second effect where we use perturbation theory to control both diagonal and off-diagonal perturbations in $F_t$
with respedt to the population trajectory.
The key step is to establish that, in a suitable matrix norm, each update
is contractive in this norm after finitely many iterations.
We separately control errors in the initial iterations before and after this contraction is reached.
For sufficiently small $\Delta_n$, an induction keeps
the empirical trajectory close enough to the population
trajectory for these local estimates to remain valid.
These estimates give a bound of
order $\Delta_n$ that holds uniformly over all iterations.
The complete proof is given in Appendix~\ref{app:tracking-proofs}.

We next translate the deterministic moment condition into a sample-size requirement.
For a real-valued random variable $Z$, define its $\psi_2$-Orlicz norm by
\[
\|Z\|_{\psi_2}
:=
\inf\left\{
s>0:
\E\exp\left(\frac{Z^2}{s^2}\right)\le 2
\right\}.
\]
A mean-zero random vector $z$ is called $K$-sub-Gaussian if
\[
\|\langle u,z\rangle\|_{\psi_2}
\le K\|u\|_2
\qquad
\text{for every }u,
\]
following the standard convention in \citep{vershynin2018high}. We next control the empirical moment error $\Delta_n$ under standard
sub-Gaussian assumptions and then apply
Theorem~\ref{thm:deterministic-tracking}.

\begin{corollary}[Sub-Gaussian sample complexity]
\label{cor:subgaussian-tracking}
Fix an integer $r_0\ge1$ and a teacher-norm bound $M_E<\infty$,
and assume
\[
1\le \rank(E)\le r_0,
\qquad
\|E\|\le M_E.
\]
Assume that the pairs $(x_i,\xi_i)$ are independent across $i$.
Suppose that each $x_i$ is isotropic and $K_x$-sub-Gaussian, and that
$\xi_i$ is independent of $x_i$, mean-zero, and has covariance
$\sigma^2I_d$.
When $\sigma>0$, assume that $\xi_i/\sigma$ is
$K_\xi$-sub-Gaussian; when $\sigma=0$, take $\xi_i=0$ almost surely.

There exist constants $C_0,C_1,C>0$, depending only on
$(r_0,K_x,K_\xi,M_E,\sigma,\lambda,\epsilon)$,
such that, for every $\eta\in(0,1)$, if
\[
n
\ge
C_0\bigl(d+\log(1/\eta)\bigr),
\]
then, with probability at least $1-\eta$,
\[
\Delta_n
\le
C_1\alpha_{n,\eta},
\qquad
\alpha_{n,\eta}
:=
\sqrt{
\frac{d+\log(1/\eta)}{n}
}.
\]
Consequently,
\begin{equation}
\sup_{t\ge0}
\max\left\{
\|B_t-\bar B_t\|,
\,
\|F_t-\bar F_t\|
\right\}
\le
C\alpha_{n,\eta}.
\label{eq:alpha-award}
\end{equation}
\end{corollary}

Thus, up to confidence and problem-dependent constants,
$n\gtrsim d$ samples are sufficient for the empirical representation and predictor to track their population trajectories uniformly over all iterations. The discrepancy decays at the usual
$n^{-1/2}$ statistical rate, despite the fact that the same data-dependent moments are reused throughout the recursion. 

\paragraph{Teacher-estimation error.}
Proposition~\ref{prop:population-dynamics} gives the exact
population error
\[
\beta_t
:=
\|\bar F_t-E\|
=
\max_{1\le i\le r}
\frac{\lambda\sqrt{\mu_i}}{b_t(\mu_i)^2+\lambda}.
\]
Under the assumptions and sample-size condition of
Corollary~\ref{cor:subgaussian-tracking}, with probability
at least $1-\eta$, the triangle inequality gives,
simultaneously for all $t\ge0$,
\begin{equation}
\|F_t-E\|
\le
\|F_t-\bar F_t\|+\beta_t
\le C\alpha_{n,\eta}+\beta_t.
\label{eq:teacher-error-path}
\end{equation}
Thus the teacher-estimation error is bounded by the explicit
population error and a sampling term that does not accumulate
across iterations. Let $\beta_\star:=\|\bar F_\star-E\|$.
Since
$\beta_t\le\beta_\star+\|\bar F_t-\bar F_\star\|$,
Corollary~\ref{cor:population-convergence} yields, for every
$\rho\in(\kappa_\star,1)$ and simultaneously for all
$t\ge T_\rho$,
\begin{equation}
\|F_t-E\|
\le
C\alpha_{n,\eta}
+
C_\rho\rho^{\,t-T_\rho}
+
\beta_\star,
\label{eq:teacher-error-decomposition}
\end{equation}
where $T_\rho,C_\rho$ are as in that corollary.
The three terms bound sampling error, population iteration error,
and regularization bias, respectively.

The bias satisfies $\beta_\star\le\lambda$.
Indeed, the limiting scalar predictor
$f_\star(\mu):=\lim_{t\to\infty}f_t(\mu)$ satisfies
\[
0\le \sqrt{\mu}-f_\star(\mu)
=
\lambda\frac{f_\star(\mu)}
{\sqrt{f_\star(\mu)^2+\epsilon}}
\le\lambda
\]
by scalar optimality. Taking the maximum over the teacher's
singular directions gives the claimed bias bound.

\section{Numerical Experiments}\label{sec:empirical-motivation}

We empirically study regularized lin-RFM on large-scale text and single-cell
gene-expression data, examining how the learned features evolve over time and how well the empirical results match theoretical findings.
There are two datasets.
On OpenWebText, a corpus of English web
documents~\citep{gokaslan2019openwebtext}, we predict a center word from
a normalized bag of surrounding words, using a one-hot vector for the
target. On the core Human Lung Cell Atlas
(HLCA)~\citep{sikkema2023integrated}, we construct a leave-one-gene-out
task, where each input is a cell's TF--IDF-weighted vector of expressions with
one expressed gene set to zero, and the target is a one-hot vector
identifying that gene. Both experiments use $d=20{,}000$ coordinates, corresponding to
words or genes. These constructions generate roughly $n=10^9$
input and prediction pairs.  
These pairs are neither independent nor isotropic in distribution, so the experiments should be viewed as empirical examinations
rather than direct tests of our theory. Deatils of data preprocessing,
sampling, and data embedding are given in the supplemental materials.

Though datasets are large, experiments are feasible because the entire RFM trajectory depends
on the data only through the empirical covariance $S_n$ and
input--output cross-moment $E_n$, as described in
Section~\ref{sec:algorithm}. We calculate these two matrices in a
streaming pass without storing the full matrices $X$ and $Y$.
All updates use only $S_n$ and $E_n$, so their cost is independent
of the number of training pairs. Table~\ref{tab:streaming-scale}
summarizes runtime costs.

\begin{table}[htbp]
\centering
\small
\caption{Total runtime split between streaming and RFM implementation.}
\label{tab:streaming-scale}
\begin{tabular}{@{}lrrp{0.24\linewidth}p{0.29\linewidth}@{}}
\toprule
Dataset & $n$ & $d$ & Streaming stage & Later RFM updates \\
\midrule
OpenWebText & $10^9$ & $20{,}000$ & 2.01 hours, 16 CPU cores & 10 updates in 1.5 min on A100 \\
HLCA & $10^9$ & $20{,}000$ & 8.51 hours, 16 CPU cores & 10 updates in 1.5 min on A100 \\
\bottomrule
\end{tabular}
\end{table}

We ideally want to track the rank of feature predictors, but due to the smoothing parameter $\epsilon > 0$ in the RFM implementation, 
learned matrices will not be exactly low rank. Every eigenvalue of
$B_t$ is at least $\epsilon^{1/4}$ which we call the \emph{smoothing floor}.
To measure spectral concentration above this baseline, we
use the stable rank $\operatorname{srank}$ defined as
\[
\widetilde B_t:=B_t-\epsilon^{1/4}I_d,
\qquad
\operatorname{srank}(\widetilde B_t)
:=
\frac{\|\widetilde B_t\|_F^2}
     {\|\widetilde B_t\|_{\mathrm{op}}^2},
\qquad
\widetilde B_t\neq0.
\]

In our experiments, $\epsilon<1$, so at initialization
$\widetilde B_0=(1-\epsilon^{1/4})I_d$ and its stable rank is $d$.
Smaller values
indicate greater concentration of the excess feature weights in
dominant directions. 
Figures~\ref{fig:owt_trifigure}
and~\ref{fig:hlca-trifigure} report the stable rank of $\widetilde B_t$ and not that of $B_t$. 

\paragraph{Results.}
Figure~\ref{fig:owt_trifigure} shows how the OpenWebText representation
changes across iterations. The first update produces pronounced
spectral weights and changes the nearest-neighbor
relationships between word embeddings. Later updates continue to reshape
the representation. Word neighborhoods are obtained from the leading spectral components of $B_t$ and compared using cosine similarity. These neighborhoods at $t=0$ appear random and form an unstructured baseline for comparison.
The sample-size panel compares trajectories computed from smaller nested
prefixes with a reference trajectory computed from a much larger sample.
The observed discrepancy decreases approximately as $n^{-1/2}$.
The reference is itself a finite-sample trajectory, not the exact population
path. These comparisons motivate the theory: the real-data covariates are sparse and anisotropic, and the theoretical
assumptions need not hold.

\section{Discussion}

Our results describe not only the limiting predictor of regularized
lin-RFM, but also the representations learned along the way. Under the
isotropic teacher model, the population trajectory admits an exact spectral
description. With sufficiently many sub-Gaussian samples, the empirical
representation and predictor remain uniformly close to this trajectory,
despite reusing the same empirical moments at every iteration. This connects
the feature-learning mechanism to its statistical stability. Three directions
would extend this picture beyond the present setting.

\begin{enumerate}

\item \textbf{Non-isotropic covariance and convergence rates.}
Our scalar reduction relies on the invariant teacher basis. For a general
positive-definite covariance $\Sigma\neq I_d$, the matrices $\Sigma$ and
$E^\top E$ need not commute, so the population representation may also
rotate. Our explicit population rate and uniform-tracking
guarantees do not directly extend.

The synthetic experiments in Figure~\ref{fig:num-trifigure} suggest an
eventual geometric tail whose contraction factor depends on the covariance
spectrum. The separation among spectral families at a fixed condition number
also suggests that condition number alone does not characterize this
dependence. Can one obtain convergence and uniform-in-time tracking bounds
that account explicitly for the spectrum of $\Sigma$ and its alignment with
$E^\top E$?

\item \textbf{Nonlinear predictors and repeated data reuse.}
An important direction is to extend our analysis beyond linear predictors.
For a differentiable activation function $\sigma$, replacing the linear predictor by
$f_t(x)=A_t\sigma(B_tx)$ makes its gradient depend on the input. The
population AGOP becomes
\[
\operatorname{AGOP}(f_t)
=
\E\!\left[
B_t^\top D_{\sigma}(B_tx)A_t^\top A_t
D_{\sigma}(B_tx)B_t
\right],
\]
where $D_\sigma(z):=\diag(\sigma'(z_j))$.
In our linear setting, $D_\sigma=I_d$, so the AGOP reduces to
$F_t^\top F_t$. With nonlinear activations, the population dynamics no
longer reduce automatically to the scalar recursion studied here, and the
empirical trajectory is generally no longer determined by the two fixed
moments $(S_n,E_n)$.

Related work provides useful partial results. Zhu et
al.~\citep{zhu2025reweighted} prove finite-sample guarantees for an iterative
kernel reweighting method using fresh samples across iterations, while Zhu
et al.~\citep{zhu2026agop} show that the AGOP of a fitted kernel predictor
can recover relevant directions in nonlinear multi-index models. A natural
next question is whether the full nonlinear RFM trajectory can be
characterized and uniformly controlled when the same finite sample is reused
at every iteration.



\item \textbf{Architectural extensions.}
AGOP-based feature learning has also been developed for more structured
RFM models, including convolutional RFMs, Deep RFMs, and tree-structured
RFMs~\citep{beaglehole2023convolutional,beaglehole2024average,
beaglehole2026xrfm}.
A natural question is whether our trajectory-level analysis extends to
these settings: can one characterize their finite-time population dynamics
and uniformly control the empirical path when the same finite sample is
reused across iterations? 
\end{enumerate}

The broader goal is to understand when repeated feature learning refines
task-relevant representations without amplifying sampling error. Our analysis
establishes this connection in a tractable linear setting; extending it
requires understanding both the population evolution and the stability of
the data-dependent path.

\begin{figure}[htbp]
    \centering
    \makebox[\linewidth][c]{%
        \includegraphics[
            width=\linewidth,
            height=0.80\textheight,
            keepaspectratio,
            trim={0 1.0in 0 0},
            clip
        ]{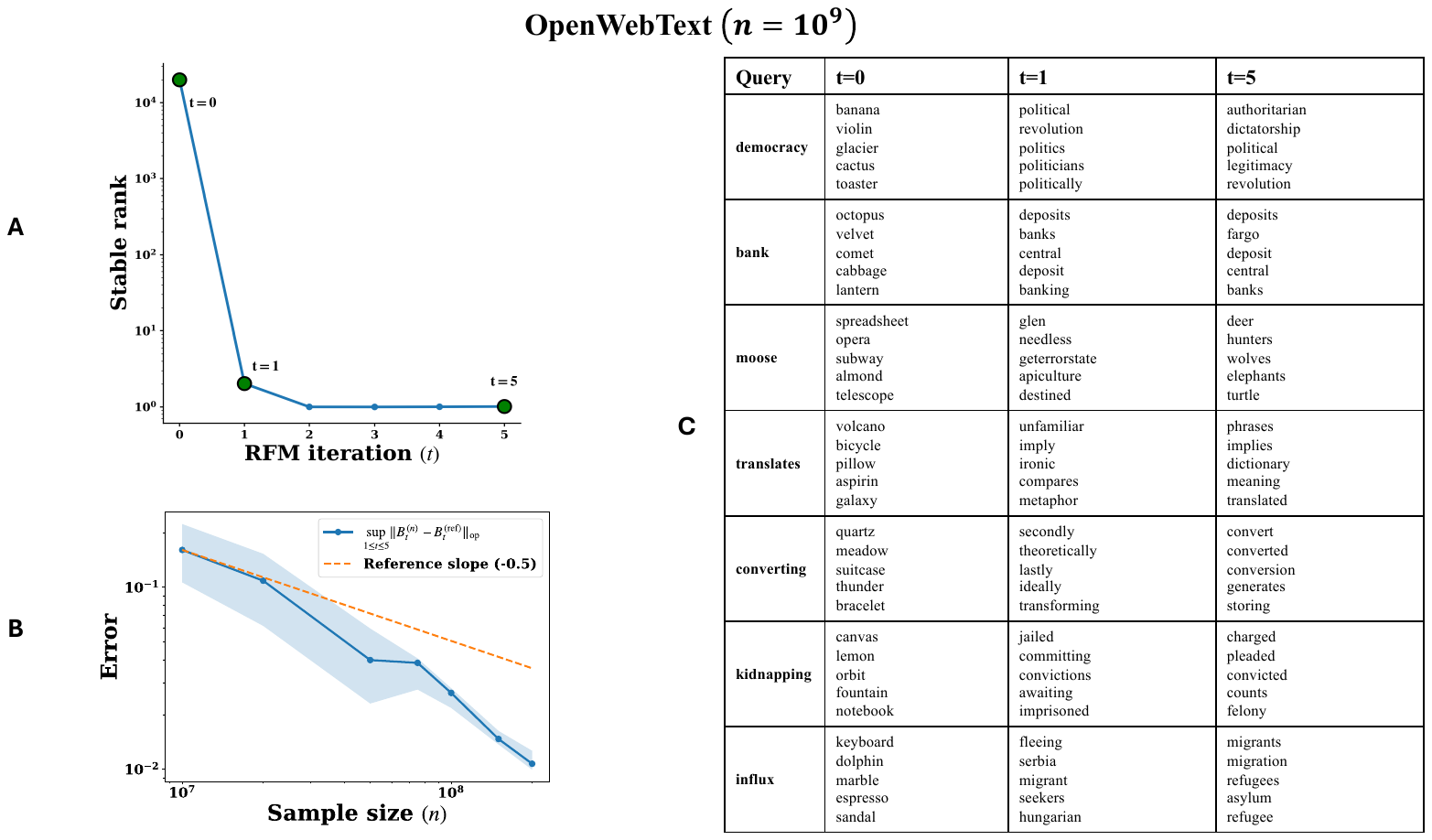}%
    }
\caption{\textbf{OpenWebText feature evolution under
regularized lin-RFM.}
\textbf{A (top-left):} Stable-rank trajectory of the
floor-subtracted feature matrix
$\widetilde B_t=B_t-\epsilon^{1/4}I_d$, highlighting
the initialization and the learned checkpoints
$t=1$ and $t=5$.
\textbf{B (bottom-left):} Maximum operator-norm discrepancy
over the evaluated iterations between finite-sample feature
matrices and a high-sample reference trajectory, plotted
against sample size $n$ and compared with an $n^{-1/2}$
reference rate.
\textbf{C (right):} Selected nearest-neighbor neighborhoods
for query words at $t=1$ and $t=5$, together with randomly
selected words at $t=0$, illustrating the task-dependent
structure that emerges across iterations.}
\label{fig:owt_trifigure}
\end{figure}

\begin{figure}[htbp]
    \centering
    \makebox[\linewidth][c]{%
        \includegraphics[
            width=\linewidth,
            height=0.80\textheight,
            keepaspectratio,
            trim={0 1.0in 0 0},
            clip
        ]{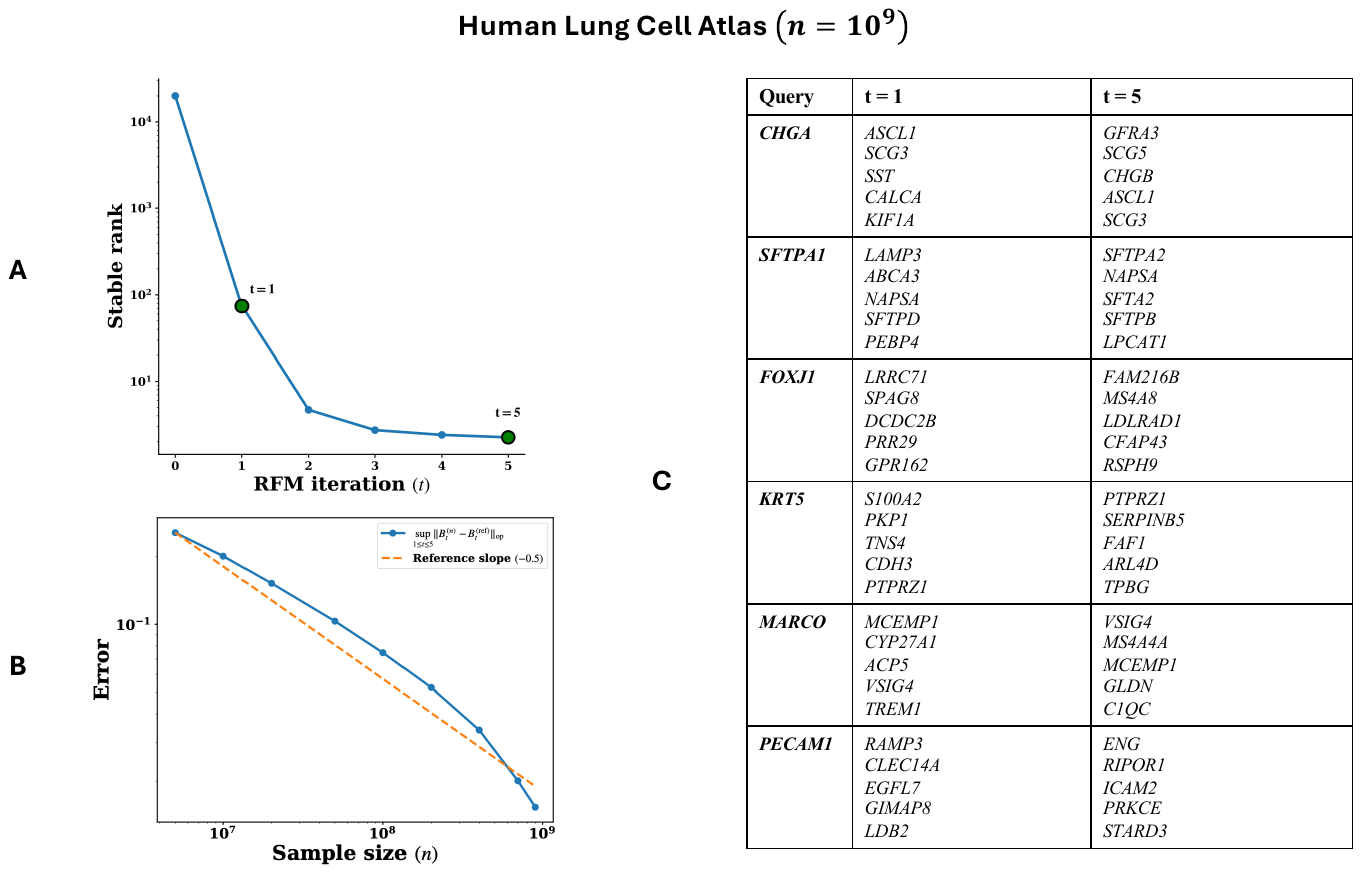}%
    }
\caption{\textbf{HLCA feature evolution under
regularized lin-RFM.}
\textbf{A (top-left):} Stable-rank trajectory of the
floor-subtracted feature matrix
$\widetilde B_t=B_t-\epsilon^{1/4}I_d$, highlighting
the initialization and the learned checkpoints
$t=1$ and $t=5$.
\textbf{B (bottom-left):} Maximum operator-norm discrepancy
over the evaluated iterations between finite-sample feature
matrices and a high-sample reference trajectory, plotted
against sample size $n$ and compared with an $n^{-1/2}$
reference rate.
\textbf{C (right):} Selected nearest-neighbor neighborhoods
for query genes at $t=1$ and $t=5$, illustrating the
task-dependent structure learned across iterations.}
\label{fig:hlca-trifigure}
\end{figure}

\begin{figure}[htbp]
\centering
\makebox[\linewidth][c]{%
    \includegraphics[
        width=\linewidth,
        trim={20bp 35bp 20bp 25bp},
        clip
    ]{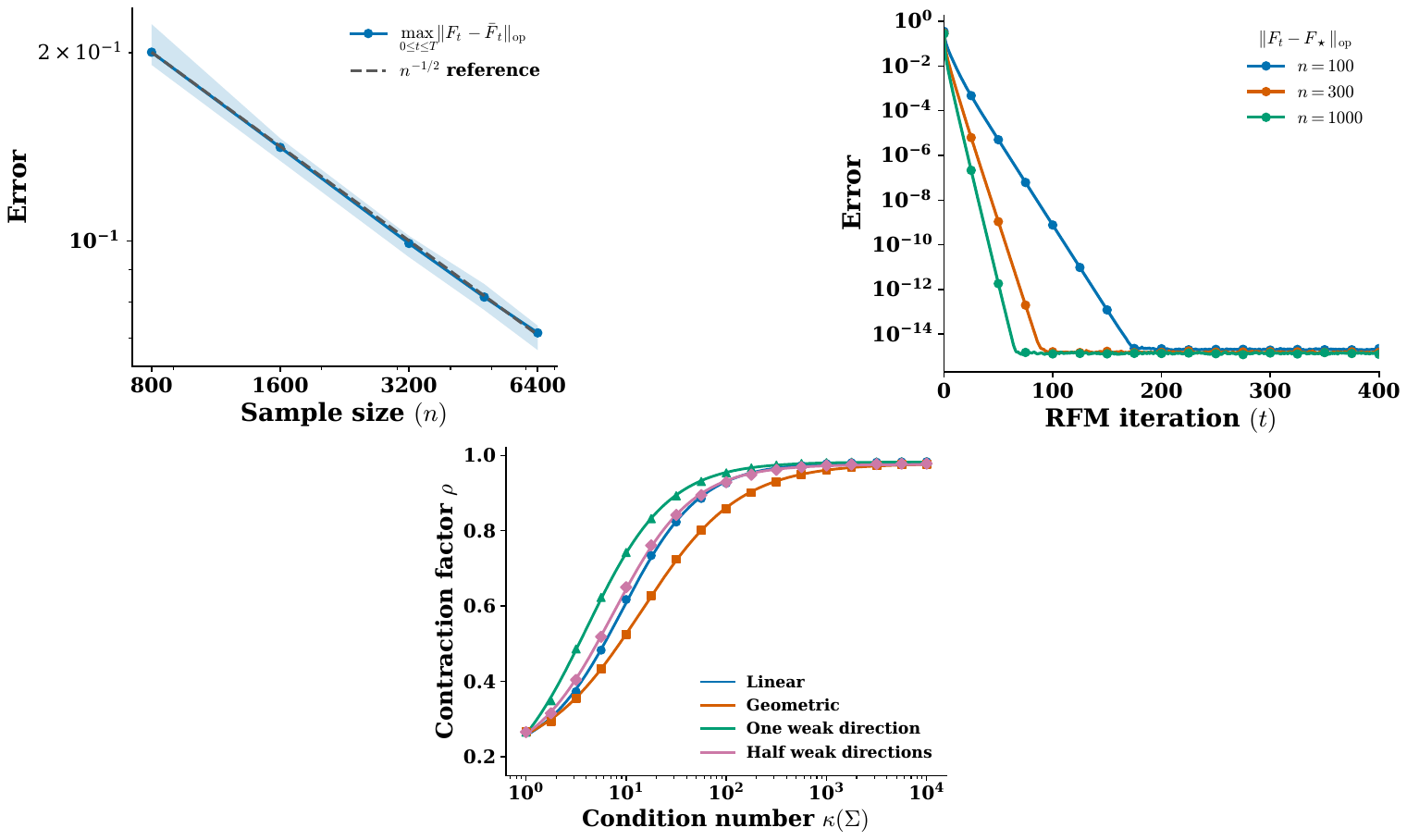}%
}
\caption{\textbf{Trajectory recovery and convergence of regularized lin-RFM.}
All panels use the model $y=Ex+\xi$ with Gaussian inputs and independent
Gaussian noise, reusing the same dataset at every update.
\textbf{Top left:}
Maximum predictor discrepancy
$\max_{0\le t\le10}\|F_t-\bar F_t\|_{\mathrm{op}}$
against sample size $n$. Points show medians over 20 independent datasets
per sample size, with a 10th--90th percentile band.
The observed decay is consistent with the $n^{-1/2}$ rate in
Corollary~\ref{cor:subgaussian-tracking}.
\textbf{Top right:}
Distance $\|F_t-F_{500}\|_{\mathrm{op}}$ to the final predictor from a
500-iteration run on the same dataset, for $n=100,300,1000$.
Curves show medians over 16 datasets per sample size and exhibit
approximately linear convergence before leveling off at numerical precision.
Both top panels use isotropic inputs, $d=p=50$, teacher rank $3$,
$(\lambda,\epsilon)=(0.3,0.04)$, and noise standard deviation $0.5$.
\textbf{Bottom:}
Estimated contraction factors $\hat\rho$ against $\kappa(\Sigma)$
for four covariance spectra: linearly spaced, geometrically spaced,
one weak direction, and half weak directions.
Each covariance is normalized to $\operatorname{tr}(\Sigma)=d$.
Rates are estimated from the late-iteration decay of the Frobenius-norm
distance to a numerical endpoint; markers show the estimates and smooth
curves are descriptive sigmoid fits.
Values closer to one indicate slower convergence. The separation between
curves shows that equal condition numbers can yield different rates
in these experiments.
Here $d=p=80$, $n=4000$, teacher rank $5$,
$(\lambda,\epsilon)=(0.01,0.001)$, and noise standard deviation $0.05$.
The covariance constructions are given in Appendix~B.3.}
\label{fig:num-trifigure}
\end{figure}

\FloatBarrier \section*{AI Disclosure}
OpenAI's ChatGPT 5.5 and Astra 6 were used to assist with manuscript
editing and exposition, including the presentation and checking of
mathematical arguments, related-work discussion, and checks of numerical
experiments. The authors reviewed and revised all AI-assisted material and
take full responsibility for the correctness and content of the paper.

\clearpage
\bibliography{main}

\clearpage
\appendix
\section*{Supplementary proofs and experiments}\label{app:supplement}

\section{Proofs for uniform finite-sample stability}
\label{app:tracking-proofs}

We prove Theorem~\ref{thm:deterministic-tracking} through the following
more general theorem. Unlike the bounded-rank statement in the main
text, this formulation allows the teacher rank to vary. The argument
is deterministic and the sub-gaussianity concentration assumptions enter only in
Section~\ref{subsec:tracking-concentration}.

\begin{theorem}[Rank-explicit uniform trajectory stability]
\label{thm:rank-explicit-tracking}
Fix $\lambda,\epsilon>0$ and $M_E<\infty$.
Let $E\in\R^{d\times d}$ satisfy
$1\le r:=\rank(E)\le d$ and $\|E\|\le M_E$.
Let $(B_t,F_t)$ follow~\eqref{eq:main-recursion} with empirical moments
$(S_n,E_n)$, and let $(\bar B_t,\bar F_t)$ follow the same recursion
with $(S_n,E_n)$ replaced by $(I_d,E)$.
Both recursions start from $B_0=\bar B_0=I_d$.
Define
\[
\Delta_n:=\|S_n-I_d\|+\|E_n-E\|.
\]
There exist constants $c,C>0$, depending only on
$(M_E,\lambda,\epsilon)$, such that
\begin{equation}
\sqrt r\,\Delta_n\le c
\quad\Longrightarrow\quad
\sup_{t\ge0}
\max\left\{
\|B_t-\bar B_t\|,
\|F_t-\bar F_t\|
\right\}
\le C\sqrt r\,\Delta_n.
\label{eq:rank-explicit-stability}
\end{equation}
\end{theorem}

\paragraph{Relation to the main theorem.}
When $r\le r_0$, Theorem~\ref{thm:deterministic-tracking} follows
immediately from Theorem~\ref{thm:rank-explicit-tracking}.
Indeed, if
\[
\Delta_n\le \frac{c}{\sqrt{r_0}},
\]
then $\sqrt r\,\Delta_n\le c$, and hence
\[
\sup_{t\ge0}
\max\left\{
\|B_t-\bar B_t\|,
\|F_t-\bar F_t\|
\right\}
\le C\sqrt r\,\Delta_n
\le C\sqrt{r_0}\,\Delta_n.
\]
Thus the constants in Theorem~\ref{thm:deterministic-tracking}
may be chosen to depend only on
$(r_0,M_E,\lambda,\epsilon)$.



\paragraph{Roadmap.}
We compare the empirical and population metrics
$H_t^{(n)}=B_t^2$ and $H_t^{(\infty)}=\bar B_t^2$ where the superscripts $(n), (\infty)$ emphasize the dependence on finite number of samples $(n)$ and the population setting, respectively. Define the trajectory error and moment error by
\[
e_t:=\|H_t^{(n)}-H_t^{(\infty)}\|_{\mathrm{blk}},
\qquad
\Delta_n:=\|S_n-I_d\|+\|E_n-E\|.
\]
The block norm is defined in Section~\ref{subsec:tracking-block-norm} and
unqualified norms are operator norms.

We first compare the empirical and population updates at the same
input, then use a Taylor expansion to control the effect of the
existing trajectory error
(Sections~\ref{subsec:tracking-maps}--\ref{subsec:tracking-remainder}).
For sufficiently small $e_t$, this gives
\[
e_{t+1}\le
\underbrace{C\Delta_n}_{\text{moment error}}
+\underbrace{\kappa_t e_t}_{\text{propagated error}}
+\underbrace{C_0\sqrt r\,e_t^2}_{\text{quadratic remainder}},
\]
where $r=\rank(E)$ and $\kappa_t$ bounds the local amplification
by the population update.

After a finite burn-in number of steps, $\kappa_t\le\rho<1$.
For sufficiently small $\sqrt r\,\Delta_n$, an induction controls
the initial growth and uses the eventual contraction to prove
$\sup_{t\ge0}e_t\le C\Delta_n$.
Converting this metric bound to operator norm and then recovering
the feature and predictor bounds proves
Theorem~\ref{thm:rank-explicit-tracking}
in Section~\ref{subsec:tracking-recursion}.
Concentration of $\Delta_n$ then gives
Corollary~\ref{cor:subgaussian-tracking}
in Section~\ref{subsec:tracking-concentration}.

Throughout, $\|\cdot\|$ denotes matrix operator norm. Fix $\lambda,\epsilon>0$, $1\le r=\rank(E)\le d$, and $\|E\|\le M_E$. Write
\[
S_n:=\frac1nXX^\top,\qquad
E_n:=\frac1nYX^\top,\qquad
\delta_S:=\|S_n-I_d\|,\qquad
\delta_E:=\|E_n-E\|,
\]
so that $\Delta_n=\delta_S+\delta_E$ is the moment error in Theorem~\ref{thm:deterministic-tracking}. Constants denoted by $c,C$ may change from line to line and depend only on $(M_E,\lambda,\epsilon)$ unless additional dependence is stated. The deterministic comparison below assumes $\delta_S\le1/2$ and $\delta_E\le1$. The hypothesis $\sqrt r\,\Delta_n\le c$ in
Theorem~\ref{thm:rank-explicit-tracking} ensures these conditions
after choosing $c\le 1/2$, since $r\ge1$ and
$\delta_S,\delta_E\le \Delta_n$. 


\subsection{Metric maps, trapping, and one-step perturbation}
\label{subsec:tracking-maps}

For $H\succ0$, define the empirical ridge factor
\[
\mathcal W_n(H):=(S_n+\lambda H^{-1})^{-1},
\]
and its population counterpart
\[
s(H):=(I_d+\lambda H^{-1})^{-1}
     =H(H+\lambda I_d)^{-1}.
\]
The corresponding empirical and population metric maps are
\begin{equation}
\begin{aligned}
T_n(H)
&:=\bigl(
\mathcal W_n(H)E_n^\top E_n\mathcal W_n(H)
+\epsilon I_d
\bigr)^{1/2},\\
T_\infty(H)
&:=\bigl(
s(H)E^\top E s(H)
+\epsilon I_d
\bigr)^{1/2}.
\end{aligned}
\label{eq:metric-maps-app}
\end{equation}
Thus, with
\[
H_0^{(n)}=H_0^{(\infty)}=I_d,
\]
the empirical and population metric trajectories satisfy
\[
H_{t+1}^{(n)}=T_n(H_t^{(n)}),
\qquad
H_{t+1}^{(\infty)}=T_\infty(H_t^{(\infty)}).
\]





\begin{lemma}[Uniform spectral bounds]\label{lem:metric-trapping}
Fix $M_E<\infty$. Suppose
\[
\|E\|\le M_E,
\qquad
\|S_n-I_d\|\le\delta_S\le\frac12,
\qquad
\|E_n-E\|\le\delta_E\le1.
\]
Set $m:=\min\{1,\sqrt\epsilon\}$ and $M:=\max\left\{1,\sqrt{4(M_E+1)^2+\epsilon}\right\}$.
Then, for every $t\ge0$,
\[
mI_d\preceq H_t^{(n)}\preceq MI_d,
\qquad
mI_d\preceq H_t^{(\infty)}\preceq MI_d.
\]
Moreover, for every $H\succ0$,
\[
0\preceq\mathcal W_n(H)\preceq S_n^{-1}\preceq2I_d,
\qquad
0\preceq s(H)\preceq I_d.
\]
\end{lemma}

\begin{proof}
Since $S_n+\lambda H^{-1}\succeq S_n\succeq I_d/2$, inversion of these inequalities and using matrix monotonicity of the inverse gives
\[
0\prec\mathcal W_n(H)\preceq S_n^{-1}\preceq2I_d.
\]
The population bound follows from $0<h/(h+\lambda)<1$ for $h>0$. For every $H\succ0$, both updates satisfy
\[
T_n(H)\succeq\sqrt\epsilon I_d,
\qquad
T_\infty(H)\succeq\sqrt\epsilon I_d.
\]
For the empirical map,
\[
\|T_n(H)\|
\le\sqrt{\|\mathcal W_n(H)\|^2\|E_n^\top E_n\|+\epsilon}
\le\sqrt{4(M_E+1)^2+\epsilon},
\]
because $\|E_n\|\le M_E+1$. The population map obeys the smaller bound $\sqrt{M_E^2+\epsilon}$. Including $H_0=I_d$ proves the claim.
\end{proof}

Define the common spectral region
\[
\mathcal K:=\{H=H^\top:mI_d\preceq H\preceq MI_d\}.
\]
The preceding lemma places both metric trajectories in $\mathcal K$
without assuming that they are close. We next compare their update maps
at a common input.

\begin{lemma}[Square-root stability]\label{lem:square-root-stability}
If $A,B\succeq\epsilon I_d$ with $\epsilon>0$, then
\[
\|A^{1/2}-B^{1/2}\|
\le\frac1{2\sqrt\epsilon}\|A-B\|.
\]
\end{lemma}

\begin{proof}
Let $X=A^{1/2}$, $Y=B^{1/2}$, and $D=X-Y$. Then
\[
XD+DY=A-B.
\]
The Sylvester equation has the integral solution
\[
D=\int_0^\infty e^{-tX}(A-B)e^{-tY}\,dt.
\]
Since $X,Y\succeq\sqrt\epsilon I_d$, taking norms and integrating gives the displayed bound.
\end{proof}

\begin{lemma}[Moment perturbation of the metric map]\label{lem:map-perturbation-H}
Under the moment conditions above, there is a constant $C=C(M_E,\lambda,\epsilon)$ such that
\begin{equation}
\sup_{mI_d\preceq H\preceq MI_d}
\|T_n(H)-T_\infty(H)\|
\le C(\delta_S+\delta_E).
\label{eq:map-perturbation-H}
\end{equation}
\end{lemma}

\begin{proof}
The resolvent identity gives the exact formula
\begin{equation}
\mathcal W_n(H)-s(H)
=\mathcal W_n(H)(I_d-S_n)s(H).
\label{eq:W-moment-identity}
\end{equation}
Hence
\[
\|\mathcal W_n(H)-s(H)\|
\le2\delta_S.
\]
Set $U_n:=E_n\mathcal W_n(H)$ and $U_\infty:=Es(H)$. Then
\[
\|U_n-U_\infty\|
\le2\delta_E+2\|E\|\delta_S,
\]
while $\|U_n\|+\|U_\infty\|$ is uniformly bounded. Since
\[
\mathcal W_n(H)E_n^\top E_n\mathcal W_n(H)=U_n^\top U_n,
\qquad
s(H)E^\top E s(H)=U_\infty^\top U_\infty,
\]
we have
\begin{align*}
\|U_n^\top U_n-U_\infty^\top U_\infty\|
&\le (\|U_n\|+\|U_\infty\|)\|U_n-U_\infty\|\\
&\le C(\delta_S+\delta_E).
\end{align*}
After adding and subtracting $\epsilon I_d$, take square roots and use Lemma~\ref{lem:square-root-stability} to get~\eqref{eq:map-perturbation-H}.
\end{proof}

\subsection{Teacher-adapted block norm}\label{subsec:tracking-block-norm}
To analyze how the population map propagates an arbitrary matrix perturbation, we use a norm adapted to the teacher's eigenspaces. This eventually makes bounding norms of perturbations possible without assuming that matrices from the empirical trajectory commute with those from the population trajectory.
Let $V=[V_r\;V_0]$ be an orthogonal eigenbasis of $E^\top E$, with $V_r$ spanning $\range(E^\top)$ and $V_0$ spanning $\ker E$. The block indices $s$ and $n$ refer to the signal and null subspaces. For any matrix $\Delta\in\R^{d\times d}$, write
\[
V^\top\Delta V=
\begin{pmatrix}
\Delta_{ss}&\Delta_{sn}\\
\Delta_{ns}&\Delta_{nn}
\end{pmatrix}
\]
and define
\begin{equation}
\|\Delta\|_{\mathrm{blk}}
:=\max\left\{
\frac1{\sqrt r}\|\Delta_{ss}\|_F,
\|\Delta_{sn}\|,
\|\Delta_{ns}\|,
\|\Delta_{nn}\|
\right\}.
\label{eq:block-norm-app}
\end{equation}
When $r=d$, the null subspace is absent and the last three entries are omitted.

\begin{lemma}[Block-norm calculus]\label{lem:block-calculus}
For every matrix $\Delta$,
\begin{equation}
\|\Delta\|_{\mathrm{blk}}
\le\|\Delta\|
\le4\sqrt r\,\|\Delta\|_{\mathrm{blk}}.
\label{eq:block-op-comparison}
\end{equation}
For arbitrary $A,B$,
\[
\|AB\|_{\mathrm{blk}}
\le4\sqrt r\,\|A\|_{\mathrm{blk}}\|B\|_{\mathrm{blk}}.
\]
If $D$ is block diagonal in the teacher basis, then
\[
\|DA\|_{\mathrm{blk}},\|AD\|_{\mathrm{blk}}
\le\|D\|\|A\|_{\mathrm{blk}}.
\]
For a square matrix $A$, one also has
\[
\|A^2\|_{\mathrm{blk}}\le2\|A\|\|A\|_{\mathrm{blk}}.
\]
\end{lemma}

\begin{proof}
Taking a submatrix does not increase operator norm and $\|\Delta_{ss}\|_F\le\sqrt r\|\Delta\|$, proving the first inequality in \eqref{eq:block-op-comparison}. Decomposing the matrix into its four embedded blocks gives
\[
\|\Delta\|
\le\|\Delta_{ss}\|+\|\Delta_{sn}\|+\|\Delta_{ns}\|+\|\Delta_{nn}\|
\le(\sqrt r+3)\|\Delta\|_{\mathrm{blk}}.
\]
Set $a:=\|A\|_{\mathrm{blk}}$ and $b:=\|B\|_{\mathrm{blk}}$. Then
\[
\|A_{ss}\|,\|B_{ss}\|\le \sqrt r\,a,\sqrt r\,b,
\]
while every other block has operator norm at most $a$ or $b$. For the signal--signal block,
\begin{align*}
\frac1{\sqrt r}\|(AB)_{ss}\|_F
&\le \frac1{\sqrt r}\|A_{ss}B_{ss}\|_F
   +\frac1{\sqrt r}\|A_{sn}B_{ns}\|_F\\
&\le \sqrt r\,ab+ab
\le 2\sqrt r\,ab.
\end{align*}
Here the second term uses $\rank(A_{sn}B_{ns})\le r$, hence $\|A_{sn}B_{ns}\|_F\le\sqrt r\|A_{sn}B_{ns}\|\le\sqrt r\,ab$.
For the signal--null block,
\[
\|(AB)_{sn}\|
\le \|A_{ss}B_{sn}\|+\|A_{sn}B_{nn}\|
\le(\sqrt r+1)ab
\le2\sqrt r\,ab.
\]
The null--signal block is identical, and
\[
\|(AB)_{nn}\|
\le \|A_{ns}B_{sn}\|+\|A_{nn}B_{nn}\|
\le2ab.
\]
Taking the maximum proves the product bound. If $D$ is block diagonal, each block of $DA$ or $AD$ contains only one block of $A$ multiplied by a block of $D$, giving the asserted multiplier bound.

For $A^2$, bound one factor in each block product by $\|A\|$. For example,
\[
\frac1{\sqrt r}\|(A^2)_{ss}\|_F
\le\|A_{ss}\|\frac{\|A_{ss}\|_F}{\sqrt r}
+\|A_{sn}\|\|A_{ns}\|
\le2\|A\|\,a.
\]
Each of the other three blocks has the same bound, proving the last inequality.
\end{proof}

Lemma~\ref{lem:map-perturbation-H} and $\|\cdot\|_{\mathrm{blk}}\le\|\cdot\|$ imply
\begin{equation}
\delta:=\sup_{H\in\mathcal K}
\|T_n(H)-T_\infty(H)\|_{\mathrm{blk}}
\le C\Delta_n.
\label{eq:block-map-delta}
\end{equation}
The supremum is over the fixed set $\mathcal K$, so this estimate applies in particular at the data-dependent input $H_t^{(n)}$.

\subsection{Elementary derivative of the population map}\label{subsec:tracking-derivative}
We now control the response of the population map to perturbations of its input. Although each population iterate is diagonal in the teacher basis, perturbations arising in the finite sample setting may not be diagonal in this basis.
Fix $H\succ0$ that is diagonal in the teacher basis. In this basis write
\[
H=\diag(h_1,\ldots,h_d),
\quad
S:=s(H)=\diag(s_1,\ldots,s_d),
\quad
s_i:=\frac{h_i}{h_i+\lambda},
\]
\[
\begin{aligned}
Q&:=SE^\top E S=\diag(q_1,\ldots,q_d),
&\quad q_i&:=\mu_i s_i^2,\\
T&:=T_\infty(H)=\diag(t_1,\ldots,t_d),
&\quad t_i&:=\sqrt{q_i+\epsilon}.
\end{aligned}
\]
The derivatives of $s(H)$ and the positive matrix square root are
instances of the Dalecki\u{i}--Kre\u{i}n formula for differentiating
functions of symmetric matrices~\citep{daletskii1965integration}.
Combining these formulas with the product and chain rules gives the
entrywise expression below. However, for sake of self-containment we include a direct derivation below.

\begin{lemma}[Entrywise population derivative]\label{lem:H-Jacobian}
For every symmetric perturbation $\Delta$,
\begin{equation}
V^\top DT_\infty(H)[\Delta]V
=K(H)\odot(V^\top\Delta V),
\label{eq:H-Jacobian-Schur}
\end{equation}
where
\begin{equation}
K_{ij}(H)
=\frac{\lambda(\mu_i s_i+\mu_j s_j)}
{(h_i+\lambda)(h_j+\lambda)(t_i+t_j)}.
\label{eq:H-Jacobian-coeff}
\end{equation}
In particular, the null--null block is annihilated. Along the population trajectory, the signal--null coefficients are constant across null columns for each fixed signal row.
\end{lemma}

\begin{proof}
The spectral function $s$ takes the form
\(
s(H)=I_d-\lambda(H+\lambda I_d)^{-1}.
\)
Differentiating the inverse gives
\[
Ds(H)[\Delta]
=\lambda(H+\lambda I_d)^{-1}\Delta(H+\lambda I_d)^{-1}.
\]
Thus, in the teacher basis,
\[
(Ds(H)[\Delta])_{ij}
=\frac{\lambda}{(h_i+\lambda)(h_j+\lambda)}\Delta_{ij}.
\]
The product rule for $Q(H)=s(H)E^\top E s(H)$ gives
\[
(DQ(H)[\Delta])_{ij}
=\frac{\lambda(\mu_i s_i+\mu_j s_j)}
{(h_i+\lambda)(h_j+\lambda)}\Delta_{ij}.
\]
Finally, differentiate the defining identity $T(H)^2=Q(H)+\epsilon I_d$. If $Z=DT(H)[\Delta]$, then
\[
TZ+ZT=DQ(H)[\Delta].
\]
Because $T$ is diagonal, $Z_{ij}=(DQ[\Delta])_{ij}/(t_i+t_j)$, which proves the formula. The block structure follows immediately from $\mu_i=0$ on the null coordinates and from equality of the null eigenvalues along the population trajectory.
\end{proof}

For a population iterate or the population fixed point, the null eigenvalues of $H$ are all equal. At such a point define
\[
\kappa(H)
:=\max\left\{
\max_{1\le i,j\le r}K_{ij}(H),
\max_{1\le i\le r}K_{i0}(H)
\right\},
\]
where $0$ denotes any null coordinate; when $r=d$, the signal--null term is omitted. The signal--null coefficients are constant across null columns, and the null--null block vanishes. Consequently, at these population points, \eqref{eq:H-Jacobian-Schur} gives
\begin{equation}
\|DT_\infty(H)[\Delta]\|_{\mathrm{blk}}
\le \kappa(H)\|\Delta\|_{\mathrm{blk}}
\label{eq:block-derivative-amplification}
\end{equation}
for every symmetric $\Delta$. Indeed, the signal--signal block is controlled entrywise in normalized Frobenius norm, while the signal--null block is left-multiplied by a diagonal matrix whose operator norm is at most $\kappa(H)$.

For later use, define
\[
h_t(\mu):=b_t(\mu)^2,
\qquad
h_\star(\mu):=b_\star(\mu)^2.
\]
Thus
\[
H_t^{(\infty)}=h_t(E^\top E),
\qquad
H_\star=h_\star(E^\top E).
\]

\begin{lemma}[Strict fixed-point contraction]\label{lem:H-fixed-contraction}
Let $H_\star=h_\star(E^\top E )$ be the population fixed point. Then
\[
0\le\kappa(H_\star)\le\bar\kappa_\star
:=\frac{\lambda}{\lambda+\sqrt\epsilon}<1.
\]
In particular, the same contraction bound holds for every teacher spectrum.
\end{lemma}

\begin{proof}
In the null coordinate, $h_{\mathrm{null}}=\sqrt\epsilon$. At a fixed point, $t_i=h_i\ge \sqrt\epsilon$ and
\[
\mu_i s_i^2=h_i^2-\epsilon,
\qquad
s_i=\frac{h_i}{h_i+\lambda}.
\]
Hence, for every coordinate, including a null coordinate,
\begin{equation}
\mu_i s_i
=\frac{(h_i^2-\epsilon)(h_i+\lambda)}{h_i}.
\label{eq:fixed-point-mu-s}
\end{equation}
Substitution into \eqref{eq:H-Jacobian-coeff} gives
\[
K_{ij}(H_\star)
=\frac{1}{h_i+h_j}
\left[
\frac{\lambda}{h_j+\lambda}
\left(h_i-\frac{\epsilon}{h_i}\right)
+
\frac{\lambda}{h_i+\lambda}
\left(h_j-\frac{\epsilon}{h_j}\right)
\right].
\]
Both terms are nonnegative. Since $h_i,h_j\ge \sqrt{\epsilon}$,
\[
\frac{\lambda}{h_j+\lambda}\le\bar\kappa_\star,
\qquad
0\le h_i-\frac{\epsilon}{h_i}\le h_i,
\]
and the same inequalities hold with $i,j$ interchanged. It follows that
\[
0\le K_{ij}(H_\star)
\le \bar\kappa_\star\frac{h_i+h_j}{h_i+h_j}
=\bar\kappa_\star<1.
\]
For completeness, a signal--null coefficient reduces to
\[
K_{i0}(H_\star)
=\frac{\lambda}{\sqrt{\epsilon}+\lambda}
\left(1-\frac{\sqrt{\epsilon}}{h_i}\right),
\]
and null--null coefficients vanish. Thus the bound controls every block entering $\kappa(H_\star)$.
\end{proof}

\begin{corollary}[Eventual derivative contraction]\label{cor:H-eventual-contraction}
Fix $\rho\in(\bar\kappa_\star,1)$. There is a finite integer $T_\rho=T_\rho(M_E,\lambda,\epsilon,\rho)\ge1$ such that
\[
\|DT_\infty(H_t^{(\infty)})[\Delta]\|_{\mathrm{blk}}
\le\rho\|\Delta\|_{\mathrm{blk}},
\qquad t\ge T_\rho.
\]
Before $T_\rho$, the derivative is bounded by a finite constant depending only on $(M_E,\lambda,\epsilon)$.
\end{corollary}





\begin{proof}
Proposition~\ref{prop:population-dynamics} gives
\[
H_t^{(\infty)}=h_t(E^\top E ),\qquad
h_0(\mu)=1,\qquad
h_{t+1}(\mu)=\phi_\mu(h_t(\mu)),
\]
where
\[
\phi_\mu(h):=\sqrt{\epsilon+\mu\left(\frac{h}{h+\lambda}\right)^2}.
\]
We need the convergence of these scalar trajectories to be uniform over $0\le\mu\le M_E^2$.

First, $h_\star(\mu)$ is continuous on this interval. Indeed, if $\mu_k\to\mu$, the fixed points $h_\star(\mu_k)$ lie in the common compact interval $[\sqrt\epsilon,\sqrt{M_E^2+\epsilon}]$. By continuity of $\phi_\mu(h)$, every subsequential limit solves the fixed-point equation at $\mu$. Uniqueness of the positive fixed point identifies this limit as $h_\star(\mu)$.

For every fixed $t$, the function $h_t(\mu)$ is also continuous, since it is obtained by finitely many iterations of the jointly continuous map $(\mu,h)\mapsto\phi_\mu(h)$. For each $\mu$, the map $\phi_\mu$ is nondecreasing.
Hence, if $h_1(\mu)\ge h_0(\mu)$, then
$h_{t+1}(\mu)\ge h_t(\mu)$ for every $t$, and similarly
if $h_1(\mu)\le h_0(\mu)$. Together with convergence to
$h_\star(\mu)$, the scalar trajectory therefore approaches its
fixed point monotonically. Hence
\[
g_t(\mu):=|h_t(\mu)-h_\star(\mu)|
\]
is a nonincreasing sequence of continuous functions converging pointwise to zero on $[0,M_E^2]$. To see that this convergence is uniform, fix $\tau>0$. The open sets $\{\mu:g_t(\mu)<\tau\}$ increase with $t$ and cover $[0,M_E^2]$. Compactness gives a finite subcover; since the sets are increasing, a single one covers the interval. Consequently
\[
\sup_{0\le\mu\le M_E^2}g_t(\mu)\longrightarrow0.
\]
The coefficient \eqref{eq:H-Jacobian-coeff} depends continuously on $\mu,\nu$ and on $h_t(\mu),h_t(\nu),h_{t+1}(\mu),h_{t+1}(\nu)$. All these quantities lie in a fixed compact set, and the denominators are bounded away from zero. Thus the time-$t$ coefficients converge uniformly on $[0,M_E^2]^2$ to their fixed-point values. Lemma~\ref{lem:H-fixed-contraction} bounds those values by $\bar\kappa_\star<\rho$, so a time $T_\rho\ge1$, depending only on $(M_E,\lambda,\epsilon,\rho)$, makes every coefficient at most $\rho$. Equation~\eqref{eq:block-derivative-amplification} proves the claim.

Finally, the trapping bounds and $0\le s_i\le1$ give, for every $t\ge0$,
\[
K_{ij}(H_t^{(\infty)})
\le L:=\frac{\lambda M_E^2}{(m+\lambda)^2\sqrt\epsilon}.
\]
This supplies the required uniform derivative bound before the burn-in.
\end{proof}

\subsection{Quadratic remainder}\label{subsec:tracking-remainder}

The derivative controls the first-order response, but the tracking argument also needs a uniform bound on the nonlinear remainder. We first treat the square root at a block-diagonal base point, then apply that estimate to the composition $T_\infty(H)=(s(H)E^\top E s(H)+\epsilon I_d)^{1/2}$. The perturbation itself need not be block diagonal.

\begin{lemma}[Block Taylor bound for the square root]\label{lem:block-square-root-Taylor}
Let $X\succeq\epsilon I_d$ be block diagonal in the teacher basis. There are constants $c,C>0$, depending only on $\epsilon$, such that if $U=U^\top$, $X+U\succeq\epsilon I_d$, and
\[
\|U\|_{\mathrm{blk}}\le\frac{c}{\sqrt r},
\]
then
\begin{align}
\|(X+U)^{1/2}-X^{1/2}\|_{\mathrm{blk}}
&\le C\|U\|_{\mathrm{blk}},\label{eq:block-sqrt-Lip}\\
\|(X+U)^{1/2}-X^{1/2}-D\sqrt{X}[U]\|_{\mathrm{blk}}
&\le C\sqrt r\,\|U\|_{\mathrm{blk}}^2.\label{eq:block-sqrt-Taylor}
\end{align}
Here $D\sqrt X[U]$ is the unique solution $Z$ of
\[
X^{1/2}Z+ZX^{1/2}=U.
\]
For every symmetric $V$, the derivative also satisfies
\[
\|D\sqrt X[V]\|_{\mathrm{blk}}
\le\frac1{2\sqrt\epsilon}\|V\|_{\mathrm{blk}}.
\]
\end{lemma}

\begin{proof}
Let $P:=X^{1/2}$. Since $P\succ0$, the inverse Sylvester operator
has the integral representation
\citep[Theorem~9.2]{bhatia1997how}
\[
\mathcal S_X(V):=\int_0^\infty e^{-tP}Ve^{-tP}\,dt.
\]
Thus $\mathcal S_X(V)$ is the unique solution $Z$ of
$PZ+ZP=V$. Differentiating the identity $(X^{1/2})^2=X$
shows that $\mathcal S_X(V)=D\sqrt X[V]$. Because $P$ is block diagonal and $P\succeq\sqrt\epsilon I_d$, the block-diagonal multiplier bound gives
\[
\|\mathcal S_X(V)\|_{\mathrm{blk}}
\le\int_0^\infty e^{-2t\sqrt\epsilon}\,dt\,
\|V\|_{\mathrm{blk}}
=\frac1{2\sqrt\epsilon}\|V\|_{\mathrm{blk}}.
\]

Set $D:=(X+U)^{1/2}-X^{1/2}$ and $u:=\|U\|_{\mathrm{blk}}$. Expanding $(P+D)^2=X+U$ gives the exact identities
\[
PD+DP=U-D^2,
\qquad
D=\mathcal S_X(U)-\mathcal S_X(D^2).
\]
Thus the nonlinear remainder is explicitly quadratic in the square-root increment $D$. To control that increment, Lemma~\ref{lem:square-root-stability} first gives
\[
\|D\|\le\frac1{2\sqrt\epsilon}\|U\|
\le\frac{2\sqrt r}{\sqrt\epsilon}\,u.
\]
Using the square bound in Lemma~\ref{lem:block-calculus}, we obtain
\[
\begin{aligned}
\|D\|_{\mathrm{blk}}
&\le\frac{u}{2\sqrt\epsilon}
+\frac1{2\sqrt\epsilon}\|D^2\|_{\mathrm{blk}}\\
&\le\frac{u}{2\sqrt\epsilon}
+\frac{\|D\|}{\sqrt\epsilon}\|D\|_{\mathrm{blk}}\\
&\le\frac{u}{2\sqrt\epsilon}
+\frac{2\sqrt r\,u}{\epsilon}\|D\|_{\mathrm{blk}}.
\end{aligned}
\]
Choose $c\le\epsilon/4$. Then $\sqrt r\,u\le c$ makes the last coefficient at most $1/2$, and moving that term to the left yields
\[
\|D\|_{\mathrm{blk}}\le\frac{u}{\sqrt\epsilon}.
\]
This proves \eqref{eq:block-sqrt-Lip}. Returning to the exact remainder identity and using the product bound gives
\[
\begin{aligned}
\|D-D\sqrt X[U]\|_{\mathrm{blk}}
&=\|\mathcal S_X(D^2)\|_{\mathrm{blk}}\\
&\le\frac{4\sqrt r}{2\sqrt\epsilon}\|D\|_{\mathrm{blk}}^2\\
&\le\frac{2\sqrt r}{\epsilon^{3/2}}\,u^2.
\end{aligned}
\]
This proves \eqref{eq:block-sqrt-Taylor}.
\end{proof}

\begin{lemma}[Local quadratic remainder]\label{lem:H-quadratic-remainder}
There exist $c_0,C_0>0$, depending only on $(M_E,\lambda,\epsilon)$, such that for every population iterate $H_t^{(\infty)}$ and every symmetric $\Delta$ satisfying
\[
\|\Delta\|_{\mathrm{blk}}\le\frac{c_0}{\sqrt r},
\]
one has
\begin{equation}
\|T_\infty(H_t^{(\infty)}+\Delta)
-T_\infty(H_t^{(\infty)})
-DT_\infty(H_t^{(\infty)})[\Delta]\|_{\mathrm{blk}}
\le C_0\sqrt r\,\|\Delta\|_{\mathrm{blk}}^2.
\label{eq:H-quadratic-remainder}
\end{equation}
\end{lemma}

\begin{proof}
Write $H:=H_t^{(\infty)}$, $R:=(H+\lambda I_d)^{-1}$, and $e:=\|\Delta\|_{\mathrm{blk}}$. By Lemma~\ref{lem:block-calculus}, $\|\Delta\|\le4\sqrt r\,e$. Choose $c_0\le m/8$ so that $H+\Delta\succeq(m/2)I_d$. Decrease $c_0$ further, if necessary, so that
\[
4\sqrt r\,e\,\|R\|\le\frac12.
\]
These choices depend only on $(M_E,\lambda,\epsilon)$.

\emph{Ridge factor.}
Set
\[
\Gamma:=s(H+\Delta)-s(H),
\qquad
A:=Ds(H)[\Delta]=\lambda R\Delta R.
\]
Since $s(H)=I_d-\lambda R$, the resolvent identity gives
\[
\Gamma=\lambda R\Delta(H+\Delta+\lambda I_d)^{-1}
=A-R\Delta\Gamma.
\]
Because $R$ is block diagonal, Lemma~\ref{lem:block-calculus} yields
\[
\|\Gamma\|_{\mathrm{blk}}
\le\lambda\|R\|^2e
+4\sqrt r\,\|R\|e\,\|\Gamma\|_{\mathrm{blk}}.
\]
The last coefficient is at most $1/2$ by the choice of $c_0$. Moving this term to the left gives $\|\Gamma\|_{\mathrm{blk}}\le Ce$. Hence
\[
s(H+\Delta)=s(H)+A+R_s,\qquad R_s=-R\Delta\Gamma,
\]
with
\[
\|A\|_{\mathrm{blk}}\le Ce,
\qquad
\|R_s\|_{\mathrm{blk}}\le C\sqrt r\,e^2.
\]

\emph{Gram matrix.} \textcolor{red}{}
Let $Q(H):=s(H)E^\top E s(H)$ and write $S:=s(H)$. Expanding exactly,
\[
Q(H+\Delta)=Q(H)+DQ(H)[\Delta]+R_Q,
\]
where
\[
DQ(H)[\Delta]=AE^\top E S+SE^\top E A,
\qquad
R_Q=R_sE^\top E S+SE^\top E R_s+(A+R_s)E^\top E (A+R_s).
\]
The matrices $S,E^\top E $ are block diagonal and uniformly bounded. Also,
\[
\|A+R_s\|_{\mathrm{blk}}
\le Ce+C\sqrt r\,e^2\le C'e
\]
under $\sqrt r\,e\le c_0$. Block-diagonal multiplication and one use of the product bound in Lemma~\ref{lem:block-calculus} now give
\[
\|DQ(H)[\Delta]\|_{\mathrm{blk}}\le Ce,
\qquad
\|R_Q\|_{\mathrm{blk}}\le C\sqrt r\,e^2.
\]

\emph{Square root.}
Set $X:=Q(H)+\epsilon I_d$ and
\[
U:=Q(H+\Delta)-Q(H)=DQ(H)[\Delta]+R_Q.
\]
The matrix $X$ is block diagonal, $X,X+U\succeq\epsilon I_d$, and $\|U\|_{\mathrm{blk}}\le C'e$. Decreasing $c_0$ once more ensures that $U$ satisfies the hypotheses of Lemma~\ref{lem:block-square-root-Taylor}. Thus
\[
\|\sqrt{X+U}-\sqrt X-D\sqrt X[U]\|_{\mathrm{blk}}
\le C\sqrt r\,e^2.
\]
The same lemma gives
\[
\|D\sqrt X[R_Q]\|_{\mathrm{blk}}\le C\sqrt r\,e^2.
\]
Finally,
\[
DT_\infty(H)[\Delta]=D\sqrt X[DQ(H)[\Delta]]
\]
by the chain rule. Substituting $U=DQ(H)[\Delta]+R_Q$ and applying the triangle inequality proves \eqref{eq:H-quadratic-remainder}.
\end{proof}

\subsection{Finite-burn-in perturbation argument}\label{subsec:tracking-recursion}
We now combine the preceding estimates. Define
\[
D_t:=H_t^{(n)}-H_t^{(\infty)},
\qquad
e_t:=\|D_t\|_{\mathrm{blk}},
\qquad e_0=0.
\]
Adding and subtracting $T_\infty(H_t^{(n)})$ gives the exact decomposition
\[
\begin{aligned}
D_{t+1}
={}&\underbrace{T_n(H_t^{(n)})-T_\infty(H_t^{(n)})}
_{\text{moment perturbation at a common input}}\\
&+\underbrace{T_\infty(H_t^{(n)})-T_\infty(H_t^{(\infty)})}
_{\text{propagation through the population map}}.
\end{aligned}
\]
The first term has block norm at most $\delta$ by \eqref{eq:block-map-delta}, since $H_t^{(n)}\in\mathcal K$. For the second term, whenever $e_t\le c_0/\sqrt r$, Lemma~\ref{lem:H-quadratic-remainder} gives
\[
T_\infty(H_t^{(n)})-T_\infty(H_t^{(\infty)})
=DT_\infty(H_t^{(\infty)})[D_t]+R_t,
\qquad
\|R_t\|_{\mathrm{blk}}\le C_0\sqrt r\,e_t^2.
\]
Consequently, the triangle inequality and \eqref{eq:block-derivative-amplification} yield
\[
\begin{aligned}
e_{t+1}
&\le \|T_n(H_t^{(n)})-T_\infty(H_t^{(n)})\|_{\mathrm{blk}}\\
&\quad+\|DT_\infty(H_t^{(\infty)})[D_t]\|_{\mathrm{blk}}
+\|R_t\|_{\mathrm{blk}},
\end{aligned}
\]
and hence
\begin{equation}
e_{t+1}\le\delta+\kappa_t e_t+C_0\sqrt r\,e_t^2,
\qquad
\kappa_t:=\kappa(H_t^{(\infty)}).
\label{eq:error-recursion}
\end{equation}
The three terms represent moment perturbation, linear propagation, and the nonlinear remainder. The estimate holds only while $e_t\le c_0/\sqrt r$; the following induction verifies this condition at every step.

\begin{lemma}[Deterministic tracking recursion]\label{lem:finite-burnin-tracking}
Let $r\ge1$, $\delta\ge0$, $c_0,C_0>0$, $L\ge0$, $0\le\rho<1$, and an integer $T_\rho\ge1$ be fixed. Suppose a nonnegative sequence $(e_t)$ satisfies $e_0=0$ and \eqref{eq:error-recursion} whenever $e_t\le c_0/\sqrt r$, with $0\le\kappa_t\le L$ for $t<T_\rho$ and $0\le\kappa_t\le\rho$ for $t\ge T_\rho$. There are constants $c,C>0$, depending only on $(L,T_\rho,\rho,C_0,c_0)$, such that
\[
\sqrt r\,\delta\le c
\qquad\Longrightarrow\qquad
\sup_{t\ge0}e_t\le C\delta.
\]
\end{lemma}

\begin{proof}
We first control the finite interval $0\le t\le T_\rho$ by induction. Set
\[
L_{\mathrm{pre}}:=L+C_0c_0,
\qquad
S_t:=\sum_{j=0}^{t-1}L_{\mathrm{pre}}^j,
\qquad S_0:=0,
\]
and let $C_{\mathrm{pre}}:=S_{T_\rho}\ge1$. The sums $S_t$ are nondecreasing. We claim that $e_t\le S_t\delta$ for every $t\le T_\rho$. The claim is true at $t=0$. Suppose it holds through time $t<T_\rho$. If
\[
\sqrt r\,\delta\le\frac{c_0}{2C_{\mathrm{pre}}},
\]
then $e_t\le C_{\mathrm{pre}}\delta<c_0/\sqrt r$, so the local recursion is valid and
\[
C_0\sqrt r e_t^2\le C_0c_0e_t.
\]
Therefore
\[
e_{t+1}\le\delta+L_{\mathrm{pre}}e_t
\le\delta+L_{\mathrm{pre}}S_t\delta
=S_{t+1}\delta.
\]
This proves the pre-burn-in claim and, at the same time, verifies the local-neighborhood condition at every induction step.

For the tail, choose $R=c_R/\sqrt r$ with
\[
0<c_R\le c_0,
\qquad
C_0c_R\le\frac{1-\rho}{2}.
\]
Whenever $t\ge T_\rho$ and $e_t\le R$, the local recursion gives
\[
e_{t+1}
\le\delta+\left(\rho+C_0\sqrt r R\right)e_t
\le\delta+a e_t,
\qquad a:=\frac{1+\rho}{2}<1.
\]
Now require, in addition,
\[
C_{\mathrm{pre}}\delta\le R,
\qquad
\delta\le(1-a)R=\frac{1-\rho}{2}R.
\]
These are both consequences of $\sqrt r\delta\le c$ for a sufficiently small constant $c$ depending only on the displayed parameters. The first condition gives $e_{T_\rho}\le R$. The second shows that $e_t\le R$ implies $e_{t+1}\le\delta+aR\le R$, so induction keeps the entire tail in the local ball. Iterating the linear inequality yields
\[
e_t\le a^{t-T_\rho}e_{T_\rho}
+\frac{1-a^{t-T_\rho}}{1-a}\delta
\le a^{t-T_\rho}C_{\mathrm{pre}}\delta
+\frac{2}{1-\rho}\delta,
\qquad t\ge T_\rho.
\]
For example, all the smallness requirements are satisfied by
\[
c:=\min\left\{
\frac{c_0}{2C_{\mathrm{pre}}},
\frac{c_R}{C_{\mathrm{pre}}},
(1-a)c_R
\right\}>0.
\]
Combining the finite and contractive phases gives the claim with $C=C_{\mathrm{pre}}+2/(1-\rho)$. Every use of the local recursion was justified by the preceding induction step.
\end{proof}

\begin{proofof}{Theorem~\ref{thm:rank-explicit-tracking}}
Fix
\[
\rho:=\frac{1+\bar\kappa_\star}{2}\in(\bar\kappa_\star,1).
\]
This choice depends only on $(\lambda,\epsilon)$, so Corollary~\ref{cor:H-eventual-contraction} supplies $L,T_\rho$ with dependence only on $(M_E,\lambda,\epsilon)$. Let $c_{\mathrm{boot}}$ be the smallness threshold in Lemma~\ref{lem:finite-burnin-tracking}, and let $C_{\mathrm{map}}\ge1$ satisfy \eqref{eq:block-map-delta}. Choose the constant $c$ in
Theorem~\ref{thm:rank-explicit-tracking} so that
\[
c\le\min\left\{\frac12,\frac{c_{\mathrm{boot}}}{C_{\mathrm{map}}}\right\}.
\]
Then $\sqrt r\,\Delta_n\le c$ implies $\delta_S\le1/2$ and $\delta_E\le1$, so Lemma~\ref{lem:metric-trapping} places both trajectories in $\mathcal K$. Moreover,
\[
\sqrt r\,\delta\le C_{\mathrm{map}}\sqrt r\,\Delta_n\le c_{\mathrm{boot}}.
\]
The local recursion \eqref{eq:error-recursion} and Lemma~\ref{lem:finite-burnin-tracking} therefore give
\[
\sup_{t\ge0}\|H_t^{(n)}-H_t^{(\infty)}\|_{\mathrm{blk}}
\le C\delta\le C\Delta_n.
\]
Converting the block norm to operator norm yields
\[
\sup_{t\ge0}\|H_t^{(n)}-H_t^{(\infty)}\|
\le C\sqrt r\,\Delta_n.
\]

For the feature maps,
$B_t=(H_t^{(n)})^{1/2}$ and
$\bar B_t=(H_t^{(\infty)})^{1/2}$.
Both metrics are bounded below by $mI_d$.
Lemma~\ref{lem:square-root-stability}, applied with this lower bound, gives
\[
\|B_t-\bar B_t\|
\le\frac1{2\sqrt m}\|H_t^{(n)}-H_t^{(\infty)}\|
\le C\sqrt r\,\Delta_n.
\]
This holds also at $t=0$, when the two metrics coincide.

For the predictors, write
\[
F_t=E_n\mathcal W_n(H_t^{(n)}),
\qquad
\bar F_t=E s(H_t^{(\infty)}).
\]

At a common $H\in\mathcal K$, \eqref{eq:W-moment-identity} gives
\[
\|E_n\mathcal W_n(H)-Es(H)\|
\le 2\delta_E+2\|E\|\delta_S\le C\Delta_n.
\]
For $H,G\in\mathcal K$, the resolvent identity also gives
\[
s(H)-s(G)
=\lambda(H+\lambda I_d)^{-1}(H-G)(G+\lambda I_d)^{-1},
\]
so
\[
\|s(H)-s(G)\|\le\frac{\lambda}{(m+\lambda)^2}\|H-G\|.
\]
Adding and subtracting $Es(H_t^{(n)})$ therefore yields
\[
\begin{aligned}
\|F_t-\bar F_t\|
&\le \|E_n\mathcal W_n(H_t^{(n)})-Es(H_t^{(n)})\|\\
&\quad+\|E\|\,\|s(H_t^{(n)})-s(H_t^{(\infty)})\|\\
&\le C\Delta_n+C\|H_t^{(n)}-H_t^{(\infty)}\|\\
&\le C\sqrt r\,\Delta_n.
\end{aligned}
\]
Taking the supremum over $t$ proves
\eqref{eq:rank-explicit-stability}, and Theorem~\ref{thm:deterministic-tracking} follows.
\end{proofof}

\subsection{Sub-Gaussian concentration}\label{subsec:tracking-concentration}
Theorem~\ref{thm:rank-explicit-tracking} reduces the statistical problem of sample complexity to controlling $\Delta_n:=\|S_n-I_d\|+\|E_n-E\|$. We obtain bounds on both terms in $\Delta_n$ from a common concentration argument.

\begin{lemma}[Concentration of the empirical moments]\label{lem:subgaussian-moments}
Under the assumptions of Corollary~\ref{cor:subgaussian-tracking}, there is a constant $C=C(K_x,K_\xi)$ such that, with probability at least $1-\eta$,
\begin{align}
\|S_n-I_d\|
&\le C\left[
\sqrt{\frac{d+\log(2/\eta)}{n}}
+\frac{d+\log(2/\eta)}{n}
\right],\label{eq:subg-cov}\\
\|E_n-E\|
&\le C(\|E\|+\sigma)
\left[
\sqrt{\frac{d+\log(2/\eta)}{n}}
+\frac{d+\log(2/\eta)}{n}
\right].\label{eq:subg-cross}
\end{align}
\end{lemma}

\begin{proof}
Define the noise cross-moment
\[
N_n:=\frac1n\sum_{i=1}^n\xi_i x_i^\top.
\]
Since $y_i=Ex_i+\xi_i$, we have the exact identity
\[
E_n-E=E(S_n-I_d)+N_n.
\]
Assume first that $\sigma>0$ and set
\[
z_i:=\xi_i/\sigma,\qquad
w_i:=\begin{pmatrix}x_i\\z_i\end{pmatrix}\in\R^{2d}.
\]
The vectors $w_i$ are independent across $i$. The assumptions on $x_i,\xi_i$ give
\[
\E[x_i z_i^\top]=0,\qquad
\E[z_i x_i^\top]=0,\qquad
\E[w_iw_i^\top]
=\begin{pmatrix}I_d&0\\0&I_d\end{pmatrix}=I_{2d}.
\]
They also have a sub-Gaussian parameter depending only on $(K_x,K_\xi)$: for $u,v\in\R^d$,
\[
\begin{aligned}
\|\langle u,x_i\rangle+\langle v,z_i\rangle\|_{\psi_2}
&\le C\bigl(K_x\|u\|_2+K_\xi\|v\|_2\bigr)\\
&\le C'(K_x,K_\xi)\|(u,v)\|_2.
\end{aligned}
\]
Apply the covariance bound for independent isotropic sub-Gaussian vectors \citep[Theorem~4.6.1]{vershynin2018high} to $(w_i)$. With probability at least $1-\eta$, the matrix
\[
A_n:=\frac1n\sum_{i=1}^n w_iw_i^\top-I_{2d}
=\begin{pmatrix}
S_n-I_d&\sigma^{-1}N_n^\top\\
\sigma^{-1}N_n&n^{-1}\sum_{i=1}^n z_i z_i^\top-I_d
\end{pmatrix}
\]
satisfies
\[
\|A_n\|\le C(K_x,K_\xi)
\left[
\sqrt{\frac{d+\log(2/\eta)}{n}}
+\frac{d+\log(2/\eta)}{n}
\right],
\]
where the factor $2$ in the dimension is absorbed into $C$. Extracting a block amounts to multiplying on the left and right by coordinate projections, each of operator norm one. Hence, on this same event,
\[
\|S_n-I_d\|\le\|A_n\|,
\qquad
\|N_n\|\le\sigma\|A_n\|.
\]
It follows that
\[
\|E_n-E\|
\le\|E\|\|S_n-I_d\|+\|N_n\|
\le(\|E\|+\sigma)\|A_n\|.
\]
These inequalities prove \eqref{eq:subg-cov} and \eqref{eq:subg-cross} simultaneously.

When $\sigma=0$, the mean-zero noise has $\E\|\xi_i\|_2^2=0$, so $\xi_i=0$ almost surely and $N_n=0$. Apply the same covariance theorem directly to $(x_i)$ in dimension $d$, then use $E_n-E=E(S_n-I_d)$. This proves both bounds in the noiseless case.
\end{proof}

\begin{proofof}{Corollary~\ref{cor:subgaussian-tracking}}
Let $D_\eta:=d+\log(2/\eta)$. On the event in Lemma~\ref{lem:subgaussian-moments},
\[
\Delta_n
\le C_{\mathrm{mom}}
\left(\sqrt{\frac{D_\eta}{n}}+\frac{D_\eta}{n}\right),
\]
where, using $\|E\|\le M_E$, the constant $C_{\mathrm{mom}}$
depends only on $(K_x,K_\xi,M_E,\sigma)$. If $n\ge D_\eta$, the second term is at most the first, so
\[
\sqrt r\,\Delta_n\le2C_{\mathrm{mom}}\sqrt{\frac{rD_\eta}{n}}.
\]
For a sufficiently large $C_0\ge1$, the condition
$n\ge C_0rD_\eta$ therefore ensures
$\sqrt r\,\Delta_n\le c$, so the rank-explicit estimate
\eqref{eq:rank-explicit-stability} applies. Substitution gives
\[
\sup_{t\ge0}\max\left\{
\|B_t-\bar B_t\|,
\|F_t-\bar F_t\|
\right\}
\le C\sqrt{\frac{rD_\eta}{n}}.
\]
Since $r\le r_0$, the condition $n\ge C_0r_0D_\eta$
implies the preceding sample-size requirement, and
the error bound is at most
$C\sqrt{r_0}\sqrt{D_\eta/n}$.
Absorbing $r_0$ into the sample-size constant and
$\sqrt{r_0}$ into the error constant gives the
bounded-rank formulation of
Corollary~\ref{cor:subgaussian-tracking}.
Since $d\ge1$, replacing $D_\eta$ by $d+\log(1/\eta)$ changes only the constants. This proves the stated sample-size requirement and rate. The concentration event concerns only the empirical moments; conditional on that event, the deterministic theorem controls every iteration. Thus no union bound over iteration times is needed.
\end{proofof}



\section{Experimental details and additional figures}\label{app:experimental-details}
\subsection{Computational complexity of regularized lin-RFM on text data}
In our OpenWebText and TinyStories experiments, each data point $x_i \in \mathbb{R}^d$ represents a sparse-vector with $k \ll d$ non-zero entries and the targets $y_i \in \mathbb{R}^d$ are one-hot encoded vectors. Our experiments are in the regime $n \gg d$ where $n$ can be as large as $10^9$ and $d$ typically around $10k$-$20k$.

Algorithm~\ref{alg:regularized-lin-rfm} requires computing the sufficient statistics $S_n := \frac{1}{n}X X^\top \in \mathbb{R}^{d \times d}$ and $ \widehat{E}_n :=\frac{1}{n}YX^\top \in \mathbb{R}^{d \times d}$. To compute $S_n, \widehat{E}$, we avoid realizing the full matrices $X, Y \in \mathbb{R}^{d \times n}$ and leverage sparsity of the vectors by streaming the data. Specifically, each update is of the form:
\[
S_n^{\mathrm(new)} \leftarrow S_n^{\mathrm{(old)}} + x_i x_i^\top
\]
where each $x_i x_i^\top$ represents an update with $k^2 \ll d^2$ non-zero entries. Hence we can compute $S_n$ in $\mathcal{O}(nk^2)$. Similarly, we compute $\widehat{E}_n$ in $\mathcal{O}(nk)$ time. Once computed, it suffices to run the updates of Algorithm~\ref{alg:regularized-lin-rfm} which requires $\mathcal{O}(d^3)$ for each iteration. Hence the total complexity of the algorithm is $\mathcal{O}(Td^3) + O(nk^2)$ where $T$ is total number of iterations. The memory complexity is $\mathcal{O}(d^2)$ as we require storing dense $d \times d$ matrices throughout the iterations.

In HuffPost News, our input $x\in \mathbb{R}^d$ is still $k$-sparse however the targets are $y \in \mathbb{R}^C$ where $C=42$ denotes the number of classes each article can be categorized in. Hence the only change is that $\widehat{E}_n \in \mathbb{R}^{C \times d}$ is no longer a square $d\times d$ matrix. Computing the sufficient statistics still require $\mathcal{O}(nk^2)$ cost. After the sufficient statistics are formed, the RFM iterations still involve
dense \(d\times d\) matrices through the regularized solve so the dominant per-iteration cost remains \(O(d^3)\). We note that \(A_tB_t\in\mathbb{R}^{C\times d}\) has rank at most \(C\ll d\).
Therefore the update
\[
B_{t+1}
=
\left((A_tB_t)^\top(A_tB_t)+\epsilon I\right)^{1/4}
\]
can be computed from the compact SVD of \(A_tB_t\), avoiding a full
\(d\times d\) eigendecomposition for this matrix power. The total cost is still
dominated by the dense RFM solves, but the low output dimension \(C=42\) makes
the feature-update step significantly cheaper in this multiclass setting. The complexity is still $\mathcal{O}(Td^3) + \mathcal{O}(nk^2).$

\subsection{Learned nearest neighbor embeddings}
\label{app:nearest_neighbor_computation}

Algorithm~\ref{alg:regularized-lin-rfm} returns the feature matrix $B$.
We inspect the learned embeddings induced by $B$ via rank truncation
and cosine similarity. Since $B$ is symmetric positive semidefinite,
write
\[
B = V \operatorname{diag}(b_1,\ldots,b_d)V^\top,
\]
where the eigenvalues are ordered $b_1\geq\cdots\geq b_d$, and define
the top-$k$ spectral embedding table induced by $B$ as
\[
Z := V_k \operatorname{diag}(b_1,\ldots,b_k)
\in \mathbb{R}^{d\times k}.
\]
The embedding of coordinate $j$ is the $j$-th row of $Z$,
$z_j:=e_j^\top Z\in\mathbb{R}^k$. We compute nearest neighbors using
cosine similarity. For query coordinate $q$ with embedding $z_q$, we
rank coordinates $j$ by
\[
\operatorname{sim}(q,j)
:=
\frac{z_q^\top z_j}
{\|z_q\|_2\|z_j\|_2}.
\]

At initialization, $B_0=I_d$, so all eigenvalues are equal and there is
no distinguished top-$k$ eigenspace. We therefore include the $t=0$
neighborhoods only as an unstructured initialization baseline and do
not interpret them as learned task-dependent structure.

\subsection{Anisotropic covariance families.}
For the anisotropic experiments in Figure~\ref{fig:num-trifigure}, we
construct $\Sigma=\diag(\lambda_1,\ldots,\lambda_d)$ with prescribed
condition number $\kappa$. We first define unnormalized eigenvalues
$a_i$ and then set
\[
\lambda_i=\frac{d\,a_i}{\sum_{j=1}^d a_j},
\]
so that $\tr(\Sigma)=d$ without changing its condition number.
The four spectral families are
\[
\begin{aligned}
\text{linear:}\quad
&a_i=1-\frac{i-1}{d-1}(1-\kappa^{-1}),\\
\text{geometric:}\quad
&a_i=\kappa^{-(i-1)/(d-1)},\\
\text{one weak direction:}\quad
&a_i=
\begin{cases}
1,&i<d,\\
\kappa^{-1},&i=d,
\end{cases}\\
\text{half weak directions:}\quad
&a_i=
\begin{cases}
1,&i\le d/2,\\
\kappa^{-1},&i>d/2.
\end{cases}
\end{aligned}
\]
Thus all four families have the same condition number and trace, but
differ in how variance is distributed across directions.

\section{Sample complexity of dynamics plots}

The sample-complexity plots examine empirically how the finite-sample
feature trajectory $B_t^{(n)}$ approaches a high-sample reference
trajectory $B_t^{(\mathrm{ref})}$. For each dataset, we fix a large
reference sample size $N_{\mathrm{ref}}$ and run the algorithm for
$t=1,\ldots,T$ on these reference statistics to obtain
$B_t^{(\mathrm{ref})}$. For varying sample sizes
$1\leq n<N_{\mathrm{ref}}$, we compute
\[
\max_{1\leq t\leq T}
\left\|B_t^{(n)}-B_t^{(\mathrm{ref})}\right\|_{\mathrm{op}}
\]
over $r$ trials and plot the median with a 10--90 percentile band.
The sampling procedure for the finite-$n$ trajectories differs across
datasets and is described below.

\subsection{OpenWebText}\label{app:openwebtext}

OpenWebText~\citep{Gokaslan2019OpenWeb} is a large-scale web text corpus introduced as an open-source replication of the WebText corpus used to train GPT-2~\citep{radford2019language}. It contains approximately $8$ million English web documents collected from Reddit-shared URLs.

\paragraph{Data preprocessing.}
 We use a simple lowercase word tokenizer: each story is lowercased, and tokens are extracted as contiguous alphabetic letters. Thus, capitalization, puncutation, numbers, and other non-letter symbols are ignored. We stream the first $250,000$ documents to count the frequency of the tokens and we keep the top $d=20k$ frequent tokens as our vocabulary while ignoring punctuations and common stopwords such as
 \[
\text{the, a, an, and, or, of, to, in, on, for, with, is, are, was, were}.
\]
Any token outside of our vocabulary is removed. For example, suppose a short document contains the token sequence
\[
(\texttt{the},\texttt{dog},\texttt{ran},\texttt{through},
  \texttt{the},\texttt{garden}).
\]
If our active vocabulary is
\[
\mathcal V=\{\texttt{dog},\texttt{ran},\texttt{garden}\},
\]
then tokens outside \(\mathcal V\), such as \(\texttt{the}\) and
\(\texttt{through}\), are removed. The filtered sequence is
\[
(\texttt{dog},\texttt{ran},\texttt{garden}).
\]


\paragraph{Multi-output regression task.}
Our regression task is using the \textit{surrounding bag-of-word context} to predict the \textit{center token}. Suppose preprocessing gives a sequence of in-vocabulary word indices 
\[
\left(w_1, w_2, \ldots, w_L\right), \quad w_i \in\{1, \ldots, d\}.
\]
Let $e_j \in \mathbb{R}^d$ be the one-hot vector for word $j$. Fix a window size $m$ and for center position $i\in\{m+1,\ldots,L-m\}$, we define the target to be the center word, i.e., $y_i := e_{w_i}$. The corresponding input is the normalized bag-of-words context surrounding the center: 
\[
x_i:=\frac{1}{2m}\sum_{\ell=1}^m
\left(e_{w_{i-\ell}}+e_{w_{i+\ell}}\right).
\]
Let $n:=L-2m$ and define
\[
X:=\begin{bmatrix}x_{m+1}&\cdots&x_{L-m}\end{bmatrix},
\qquad
Y:=\begin{bmatrix}y_{m+1}&\cdots&y_{L-m}\end{bmatrix}.
\]
This yields the data set $(X,Y) \in \mathbb{R}^{d\times n} \times \mathbb{R}^{d\times n}.$ We apply Algorithm~\ref{alg:regularized-lin-rfm} on this fixed dataset to obtain learned matrices $B, F$.

\paragraph{Sample complexity plots}
For OpenWebText-simple, each sample is a word-level context prediction example. We use a fixed vocabulary and form a normalized bag-of-words context vector \(x_i\) from the surrounding words, while the target \(y_i\) is the one-hot encoding of the center word. The reference statistics are computed from the largest stream used in the experiment, with \(N_{\mathrm{ref}}=10^9\) context-center examples. The finite-\(n\) prefixes are sampled from the same data stream used to construct the \(10^9\)-sample reference, so all finite samples are subsets of the reference stream. We use \(r=8\) trials. We compare \(B_t^{(n)}\) to the \(10^9\)-sample reference trajectory and plot the median feature error with a 10--90 percentile band.

\paragraph{Specific parameters.}
In our OpenWebText experiment, we use $\lambda=\epsilon=10^{-3}$
and a context window of size $m=12$. We collect $n=10^9$ samples
and use a vocabulary of $d=20{,}000$ words constructed as described
above. To construct the spectral embedding table induced by $B$
(Appendix~B.2), we truncate to the top $k=2048$ eigenvalues.

\paragraph{Computational time.}
We set \(d=20\text{k}\). Streaming the statistics \(S_n\) and
\(\widehat E_n\) for \(n=10^9\) took 2.01 hours on a compute node
using 16 CPU cores. The node contained dual Intel Xeon Platinum 7268
processors at 2.90 GHz with 48 physical cores. No GPUs were used to
compute these statistics. The regularized lin-RFM requires matrix
operations on \(20\text{k}\times20\text{k}\) matrices. We used an
NVIDIA A100-SXM4 GPU with 80 GB of memory to compute 10 iterations of
regularized lin-RFM, which took 1.5 minutes.

\subsection{Integrated Human Lung Cell Atlas (HLCA)}\label{app:hlca}
The integrated Human Lung Cell Atlast v1.0 (HLCA)~\citep{sikkema2023integrated} is the first large-scale integrated single-cell reference atlast of the human lung. We use the core integrated Human Lung Cell Atlas (HLCA) v1.0~\citep{hlca2023hcaportal}. This HLCA contains \(584,000\) healthy lung tissue cells from 107 individuals.

\paragraph{Data preprocessing.} 
We restrict our gene vocabulary $\mathcal{V}$ to be the top $d=20k$ frequent genes via \textit{detection frequency}, i.e., by genes that are expressed in the largest number of cells. We filter out cells with less than $50$ total gene counts. Then, for each cell $c$ let $z_{c,g}$ denote the read count for gene $g$ in cell $c$. We $l1$-normalize the counts $z_{c,g}$:

\[
\mathbf{t f}_{c, g}=\frac{z_{c, g}}{\sum_{h \in \mathcal{V}} z_{c, h}}
\]
Next, we apply TF-IDF (Term Frequency-Inverse Document Frequency). Scale by 
\[\mathbf{tdf}_g \approx \log \frac{\# \text { cells }}{\# \text { cells expressing gene } g}\]
and construct the vectors 
\[
v_{c,g} = \mathbf{t f}_{c, g} \cdot \mathbf{tdf}_g
\]

\paragraph{Multi-output regression task.}
For each cell $c$ define $v_c \in \mathbb{R}^d$ as the normalized expression vector over the selected vocabulary of $d$ genes. Define $J_c := \{g : v_{c,g} > 0\}$ as the set of expressed genes in cell $c$. For each $g \in \mathcal{J}_c$ we construct the leave-one-out-prediction task as 
\[
x_{c,g} = v_c - v_{c,g}e_g, \quad \text{and} \quad y_{c,g} = e_g.
\]
Thus $x_{c,g}$ contains the remaining expression profile of
cell $c$, and the target $y_{c,g}=e_g$ encodes the identity
of the held-out expressed gene. We construct our sufficient statistics as 
\[
S_n=\frac{1}{n} \sum_c \sum_{g \in J_c} x_{c, g} x_{c, g}^{\top} \qquad \widehat{E}_n =\frac{1}{n} \sum_c \sum_{g \in J_c} y_{c, g} x_{c, g}^{\top},
\]
where $n=\sum_c\left|J_c\right|$. Hence the HLCA experimental set-up is analogous to the OpenWebText context-to-center prediction task (Appendix~\ref{app:openwebtext}) except that the cells are now unordered bags of expressed genes rather than sequence of tokens. 

\paragraph{Sample complexity plots.}
For HLCA, the samples are the induced leave-one-gene-out pairs from cells. Each cell \(c\) has a TF-IDF vector \(z_c\in\mathbb R^{20000}\). For each expressed selected gene \(g\in J_c\), we form

\[
x_{c,g}=z_c-z_{c,g}e_g,
\qquad
y_{c,g}=e_g.
\]

Thus the target is one-hot. The reference statistics are computed from the HLCA data used in our
experiment, giving \(N_{\mathrm{ref}}=1.003\times10^9\)
leave-one-gene-out pairs. For finite \(n\), each trial samples a random permutation of cells without replacement and takes nested prefixes. Since each cell contributes \(|J_c|\) leave-one-gene-out pairs, checkpoints are saved when the cumulative number of pairs first exceeds the target \(n\). We use \(r=5\) trials and plot the median across trials with a 10--90 percentile band.

\paragraph{Nearest neighbors.}
After running Algorithm~\ref{alg:regularized-lin-rfm} on the constructed regression task we obtain $B_t$. For each query gene we obtain its neighbors via the method in Appendix~\ref{app:nearest_neighbor_computation} with $k=256.$ We apply a post-processing step via co-occurence filter when construting our neighborhood set $N_g$. Specifically for gene $g$ we only keep candidate neighbor genes $h$ that co-occurred with $g$ in at least $50$ cells, i.e., $\# \{c: z_{c,g} > 0 , z_{c,h} > 0\} \geq 50.$


\paragraph{Computational trick.}
Let $k_c:=|J_c|$ and let
$u_c:=\sum_{g\in J_c}e_g$ denote the support indicator of $v_c$.
For the inputs $x_{c,g}=v_c-v_{c,g}e_g$ and one-hot targets
$y_{c,g}=e_g$, summing over all held-out genes in a cell gives
\[
\begin{aligned}
\sum_{g\in J_c}x_{c,g}x_{c,g}^{\top}
&=(k_c-2)v_cv_c^{\top}
  +\operatorname{diag}(v_c^2),\\
\sum_{g\in J_c}y_{c,g}x_{c,g}^{\top}
&=u_cv_c^{\top}
  -\operatorname{diag}(v_c),
\end{aligned}
\]
where $v_c^2$ denotes the entrywise square.
Thus the two moment sums can be accumulated using the sparse
outer products $v_cv_c^\top$ and $u_cv_c^\top$, together with
diagonal corrections, in $O(k_c^2)$ operations per cell,
without explicitly constructing the leave-one-out pairs.

\paragraph{Computational time.}
We set \(d=20k\) and streamed \(n=1.003\times 10^9\)
leave-one-gene-out pairs to compute the sufficient statistics
\(S_n\) and \(\widehat{E}_n\). This took 8.51 hours on a compute
node using 16 CPU cores. The node contained dual Intel Xeon
Platinum 7268 processors at 2.90 GHz with 48 physical cores.
No GPUs were used to compute these statistics.

The runtime is substantially higher than for OpenWebText because
each cell contains an average of 1933 nonzero vocabulary genes.
Consequently, the sparse outer-product computation for a cell
involves roughly \(1933^2 \approx 3.7\) million nonzero entries.
In contrast, the OpenWebText experiment uses a context window of
size \(12\), so each input vector has at most \(24\) nonzero
coordinates.

\subsection{TinyStories}
TinyStories is a moderate-scale english text dataset comprising of $2$-$3$ million short children-style stories generated by GPT-3.5 and GPT-4~\citep{eldan2023tinystories}. 

\paragraph{Data preprocessing.} We streamed $200k$ stories and kept the top $d=10k$ frequent tokens as our vocabulary set. We use a simple lowercase word tokenizer: each story is lowercased, and tokens are extracted as contiguous alphabetic letters. Thus, capitalization, puncutation, numbers, and other non-letter symbols are ignored.

\paragraph{Multi-output regression task.} This set-up is the same \textit{surrounding bag-of-words context to predict center word} as the OpenWebText setting (see Appendix~\ref{app:openwebtext}). Our statistical feature recovery experiment (see Figure~\ref{fig:tinystories_feature_tracking_vs_n}) aligns with the results of Theorem~\ref{thm:deterministic-tracking} despite being beyond the theorem's assumptions.

\paragraph{Sample complexity plot.}
Each sample is a context-to-center-token prediction example. We fix a
word vocabulary and construct \(x_i\) as the normalized bag of surrounding
context words, while \(y_i\) is the one-hot encoding of the center word.
The reference trajectory is computed from
\(N_{\mathrm{ref}}=10^7\) context-center examples.

For each trial, we randomly permute the \(N_{\mathrm{ref}}\) reference
examples without replacement and construct the finite-sample datasets
using nested prefixes of this permutation. Thus, within a trial, the
dataset at a smaller sample size is contained in the dataset at every
larger sample size. For each sample size \(n\), we compute
\[
\Delta_B(n)
=
\max_{1\le t\le T}
\left\|
B_t^{(n)}-B_t^{(\mathrm{ref})}
\right\|_{\mathrm{op}}.
\]
We plot the median across trials together with a 10--90 percentile band.

\paragraph{Hyperparameters.}
We use $\lambda, \epsilon = 1e$-3 and context window of $m=12.$ We collected $n=10^8$ samples to train Algorithm~\ref{alg:regularized-lin-rfm}. To construct the spectral embedding table induced by $B$ (Appendix~\ref{app:nearest_neighbor_computation}) we truncate to the top $k=128$ eigenvalues.

\begin{figure}[htbp]
\centering
\makebox[\textwidth][c]{%
\begin{minipage}{0.97\textwidth}
\centering

\subfigure[{Stable-rank trajectory of
\(\widetilde B_t = B_t-\epsilon^{1/4}I\).\label{fig:tinystories_stable_rank_panel}}][b]{%
\begin{minipage}[b]{0.45\linewidth}%
\centering
\vspace{0pt}
\includegraphics[
    width=\linewidth,
    trim=0.05cm 0.02cm 0.05cm 0.02cm,
    clip
]{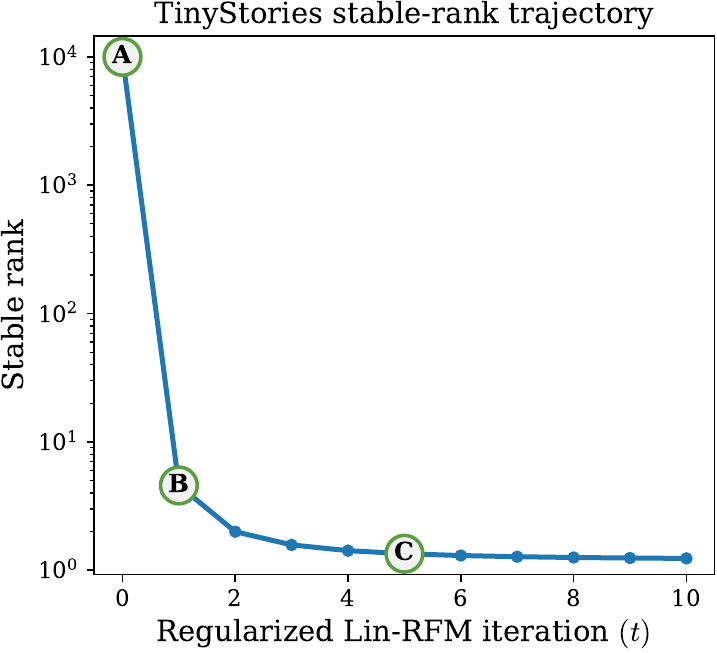}
\vspace{-0.2em}
\end{minipage}%
}
\hfill
\subfigure[{Top-5 nearest neighbors at checkpoints \(A,B,C\).\label{fig:tinystories_neighbors_panel}}][b]{%
\begin{minipage}[b]{0.53\linewidth}%
\centering
\vspace{0pt}
{\renewcommand{\arraystretch}{0.92}
\resizebox{\linewidth}{!}{%
\begin{tabular}{@{}l@{\quad}l@{\quad}l@{\quad}l@{}}
\toprule
\textbf{Query}
& \textbf{A (T=0)}
& \textbf{B (T=1)}
& \textbf{C (T=5)} \\
\midrule
\textbf{crayon}
& \makecell[l]{larry (1.00) \\ tons (1.00) \\ hospitality (1.00) \\ boost (1.00) \\ acts (1.00)}
& \makecell[l]{card (0.825) \\ sticker (0.806) \\ crown (0.797) \\ draws (0.769) \\ hat (0.759)}
& \makecell[l]{drew (0.650) \\ write (0.618) \\ circle (0.595) \\ draw (0.582) \\ thread (0.580)}
\\
\midrule
\textbf{monster}
& \makecell[l]{ozzie (1.00) \\ express (1.00) \\ weekend (1.00) \\ avocado (1.00) \\ shall (1.00)}
& \makecell[l]{wolf (0.859) \\ storm (0.845) \\ dragon (0.836) \\ giant (0.817) \\ tiger (0.815)}
& \makecell[l]{scared (0.668) \\ followed (0.653) \\ scary (0.630) \\ fire (0.629) \\ behind (0.616)}
\\
\midrule
\textbf{flower}
& \makecell[l]{postmaster (1.00) \\ facts (1.00) \\ customers (1.00) \\ result (1.00) \\ model (1.00)}
& \makecell[l]{star (0.815) \\ leaf (0.735) \\ butterfly (0.723) \\ caterpillar (0.717) \\ rainbow (0.709)}
& \makecell[l]{soft (0.667) \\ pink (0.629) \\ bee (0.619) \\ made (0.615) \\ princess (0.603)}
\\
\midrule
\textbf{sleep}
& \makecell[l]{violins (1.00) \\ approaching (1.00) \\ fewer (1.00) \\ haven (1.00) \\ lincoln (1.00)}
& \makecell[l]{stay (0.687) \\ bed (0.576) \\ relax (0.575) \\ get (0.572) \\ school (0.562)}
& \makecell[l]{nap (0.640) \\ brush (0.632) \\ asleep (0.605) \\ bed (0.575) \\ leave (0.573)}
\\
\midrule
\textbf{dragon}
& \makecell[l]{tom (1.00) \\ discuss (1.00) \\ flopped (1.00) \\ reveal (1.00) \\ nails (1.00)}
& \makecell[l]{lion (0.932) \\ knight (0.898) \\ horse (0.878) \\ prince (0.856) \\ frog (0.854)}
& \makecell[l]{bunny (0.754) \\ duck (0.747) \\ rabbit (0.705) \\ bear (0.703) \\ lion (0.701)}
\\
\bottomrule
\end{tabular}
}
}
\vspace{-0.10em}
\end{minipage}%
}

\end{minipage}
}

\vspace{-0.25em}
\caption{
\textbf{Left:} Stable-rank trajectory of
\(\widetilde B_t = B_t-\epsilon^{1/4}I\), with checkpoints \(A,B,C\).
\textbf{Right:} nearest-neighbor neighborhoods at those checkpoints.
}
\label{fig:tinystories_stable_rank_neighbors}
\end{figure}



\begin{figure}[htbp]
\centering
\includegraphics[
    width=0.72\textwidth,
    trim=0.05cm 0.05cm 0.05cm 0.05cm,
    clip
]{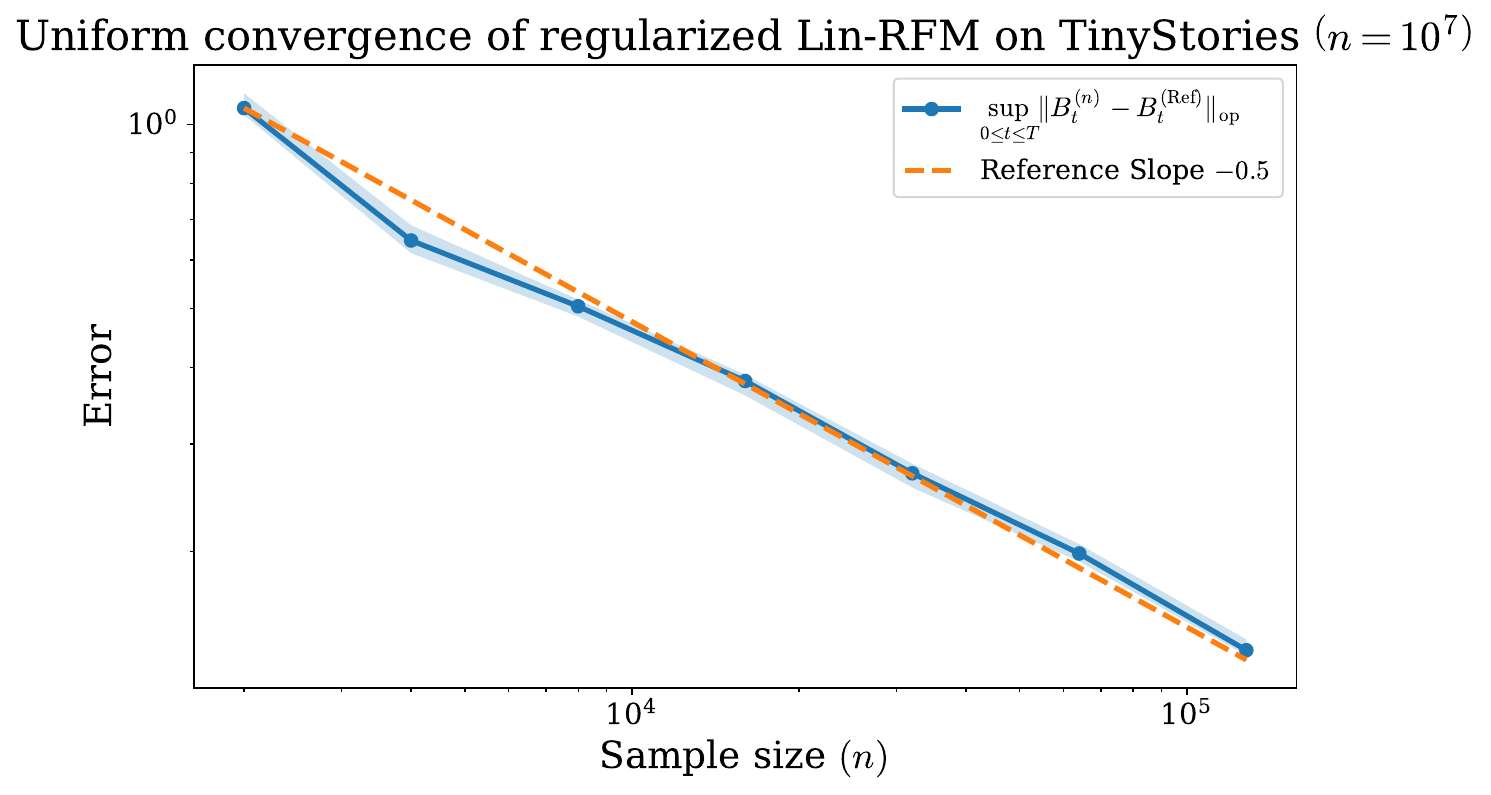}
\vspace{-0.35em}
\caption{
\textbf{TinyStories feature tracking.}
We first compute a reference trajectory \(B_t^{(\mathrm{Ref})}\) using \(n_{\mathrm{ref}}=10^7\) TinyStories context-to-center samples. For each smaller sample size \(n\), we draw random subsamples from these \(10^7\) examples without replacement, run regularized Lin-RFM for \(T=10\) iterations, and report
\(\sup_{0\le t\le T}\|B_t^{(n)}-B_t^{(\mathrm{Ref})}\|_{\mathrm{op}}\).
The empirical log-log slope is close to the reference \(n^{-1/2}\) rate, even though the TinyStories design is sparse, highly anisotropic, and outside the assumptions of Theorem~\ref{thm:deterministic-tracking}.
}
\label{fig:tinystories_feature_tracking_vs_n}
\end{figure}

\subsection{HuffPost News}\label{app:huffpost_news}
The HuffPost News dataset~\citep{misra2022news} consits of 209,527 Huffington Post news articles. Each article consists of a news headline, a summary of the article, and the category of the article. There are $C=42$ categories covering topics like comedy, media, money, travel, etc.

\paragraph{Data preprocessing.}
For each article $i$ in HuffPost News, we formed the input text $x_i \in \mathbb{R}^d$ by concatenating the headline and summary: 
\[
\operatorname{text}_i = \text{headline}_i + \text{summary}_i.
\]
We enforced the text to be lower-cased and removed the following stopwords:
\[
\text{the, a, an, and, or, of, to, in, on, for, with, is, are, was, were}
\]
to further remove words that carry little category-specific information. We apply the TF-IDF vectorizer to obtain $x_i = \text{TF-IDF(text}_i)$. Our targets $y_i \in \mathbb{R}^C$ are one-hot encoded vectors corresponding to the category that article $i$ lies in.

\paragraph{Multi-output regression task and other experiments.}
We perform an 80/20 train-test split with $n_\mathrm{train} = 167,621$ articles and $n_\mathrm{test} = 41,906$. We perform Algorithm~\ref{alg:regularized-lin-rfm} on $(X_\mathrm{train}, Y_\mathrm{train})$ and evaluate on $(X_\mathrm{test}, Y_\mathrm{test})$. Figure~\ref{fig:huffpost_test_mse} shows that subsequent lin-RFM iterations reduce test MSE
relative to the initial ridge predictor. Our nearest neighbor embeddings (Figure~\ref{fig:huffpost_stable_rank_neighbors}) and feature recovery experiment (Figure~\ref{fig:huffpost_feature_tracking}) are obtained by fitting our algorithm on all 209,527 articles.

\paragraph{Sample complexity plot.}
For HuffPost, each sample is an article-level supervised learning example.
We construct a fixed TF-IDF feature vector \(x_i\) from the article text,
and the target \(y_i\) is the one-hot encoding of the article category.
The reference statistics are computed using the full HuffPost dataset of
\(N_{\mathrm{ref}}=209{,}527\) articles. For each trial, we randomly
permute the articles and use nested prefixes for the finite sample sizes
\(n<N_{\mathrm{ref}}\). We compute the RFM trajectory from the
finite-sample sufficient statistics and compare it with the full-data
reference trajectory. As above, the plotted quantity is the median over
trials of
\[
\max_{1\le t\le T}
\left\|B_t^{(n)}-B_t^{(\mathrm{ref})}\right\|_{\mathrm{op}},
\]
with a 10--90 percentile band.

\paragraph{Hyperparameters.}
We set $\lambda = \epsilon = 0.001$. The TF-IDF vectorizer hyperparameters are:
\texttt{lowercase}=True,
\texttt{min\_df}=5,
\texttt{max\_df}=0.70,
and \texttt{ngram\_range}=(1,1). To construct the spectral embedding table
(see Appendix~\ref{app:nearest_neighbor_computation}),  we truncate to the top $k=64$ singular values.

\begin{figure}[htbp]
\centering
\includegraphics[
    width=0.70\textwidth,
    trim=0.05cm 0.05cm 0.05cm 0.05cm,
    clip
]{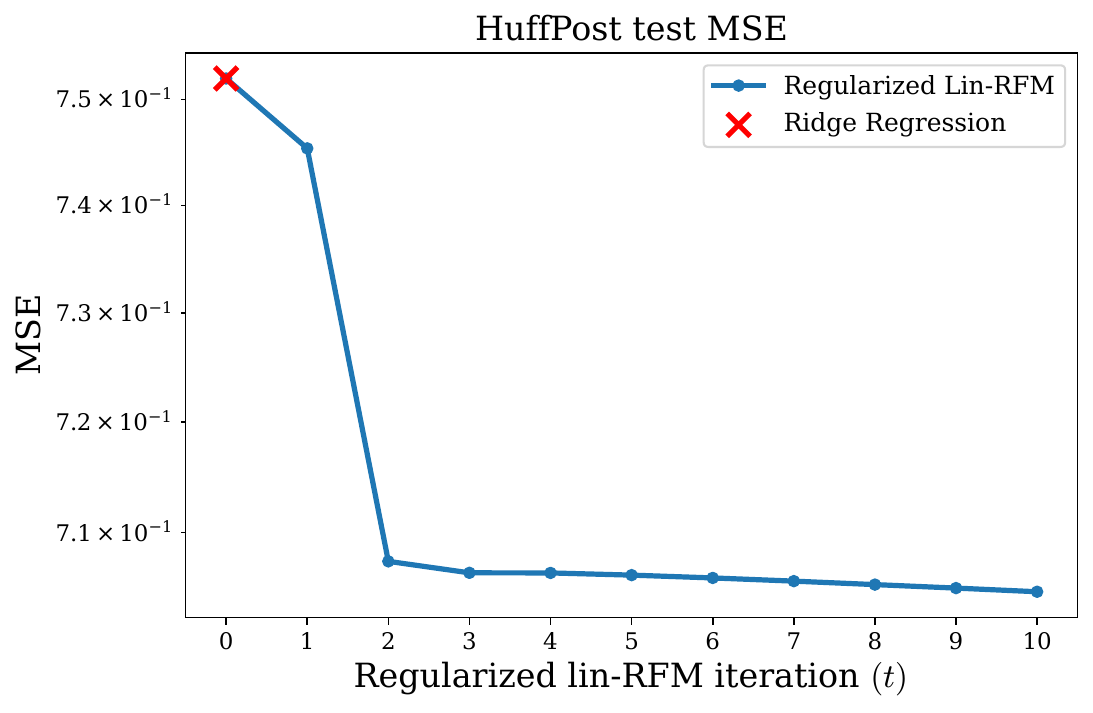}

\vspace{-0.35em}
\caption{
\textbf{HuffPost News test MSE.} 80/20 train-test split with $n_\mathrm{train} = 167,621$ articles and $n_\mathrm{test} = 41,906$. We perform  Algorithm~\ref{alg:regularized-lin-rfm} on $(X_\mathrm{train}, Y_\mathrm{train})$ and evaluate on $(X_\mathrm{test}, Y_\mathrm{test})$. We observe that our algorithm learns to generalize with respect to the MSE loss in Figure~\ref{fig:huffpost_test_mse}.
The ridge baseline is the initial predictor at $t=0$; subsequent
lin-RFM iterations reduce test MSE relative to this initial ridge fit.
}
\label{fig:huffpost_test_mse}
\end{figure}

\begin{figure}[htbp]
\centering
\makebox[\textwidth][c]{%
\begin{minipage}{0.97\textwidth}
\centering

\subfigure[{Stable-rank trajectory of
\(\widetilde B_t = B_t-\epsilon^{1/4}I\)..\label{fig:huffpost_stable_rank_panel}}][b]{%
\begin{minipage}[b]{0.41\linewidth}%
\centering
\vspace{0pt}
\includegraphics[
    width=\linewidth,
    trim=0.05cm 0.02cm 0.05cm 0.02cm,
    clip
]{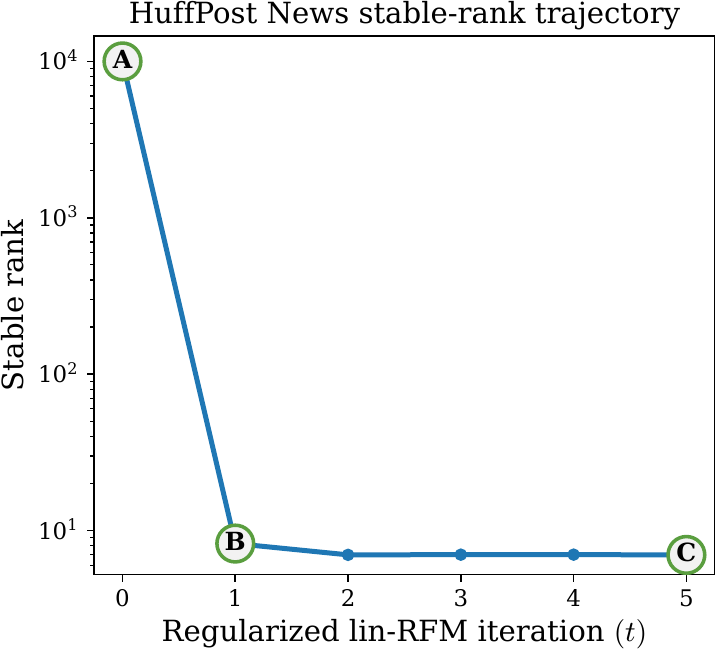}
\vspace{-0.2em}
\end{minipage}%
}
\hfill
\subfigure[{Top-5 nearest neighbors at checkpoints \(A,B,C\).\label{fig:huffpost_neighbors_panel}}][b]{%
\begin{minipage}[b]{0.56\linewidth}%
\centering
\vspace{0pt}
{\renewcommand{\arraystretch}{0.88}
\resizebox{\linewidth}{!}{%
\begin{tabular}{@{}l@{\quad}l@{\quad}l@{\quad}l@{}}
\toprule
\textbf{Query}
& \textbf{A (T=0)}
& \textbf{B (T=1)}
& \textbf{C (T=5)} \\
\midrule
\textbf{campaign}
& \makecell[l]{larry (1.00) \\ tons (1.00) \\ hospitality (1.00) \\ boost (1.00) \\ acts (1.00)}
& \makecell[l]{clinton (0.977) \\ marco (0.969) \\ senator (0.972) \\ republican (0.972) \\ jeb (0.967)}
& \makecell[l]{secretary (0.983) \\ lawmaker (0.977) \\ caucus (0.975) \\ rep (0.983) \\ republican (0.983)}
\\
\midrule
\textbf{iran}
& \makecell[l]{calls (1.00) \\ skill (1.00) \\ rug (1.00) \\ squeezes (1.00) \\ chicky (1.00)}
& \makecell[l]{netanyahu (0.958) \\ putin (0.956) \\ troops (0.943) \\ kurds (0.882) \\ palestinian (0.926)}
& \makecell[l]{netanyahu (0.985) \\ afghanistan (0.980) \\ israelis (0.870) \\ troops (0.971) \\ nuclear (0.962)}
\\
\midrule
\textbf{sports}
& \makecell[l]{tom (1.00) \\ discuss (1.00) \\ flopped (1.00) \\ reveal (1.00) \\ nails (1.00)}
& \makecell[l]{49ers (0.955) \\ nba (0.981) \\ basketball (0.976) \\ broncos (0.946) \\ nfl (0.980)}
& \makecell[l]{coach (0.990) \\ basketball (0.987) \\ hockey (0.975) \\ seahawks (0.977) \\ broncos (0.965)}
\\
\midrule
\textbf{gender}
& \makecell[l]{violins (1.00) \\ approaching (1.00) \\ fewer (1.00) \\ haven (1.00) \\ remote (1.00)}
& \makecell[l]{equality (0.923) \\ sexuality (0.855) \\ bi (0.793) \\ pride (0.848) \\ rupaul (0.845)}
& \makecell[l]{equality (0.957) \\ discriminate (0.849) \\ sexuality (0.924) \\ identity (0.920) \\ laverne (0.906)}
\\
\bottomrule
\end{tabular}
}
}
%
%
\end{minipage}%
}

\end{minipage}
}

\caption{
\textbf{Left:} Stable-rank trajectory of
\(\widetilde B_t = B_t-\epsilon^{1/4}I\)., with checkpoints \(A,B,C\).
\textbf{Right:} nearest-neighbor neighborhoods induced by $B_t$ with cosine similarity scores at those checkpoints (see Appendix~\ref{app:nearest_neighbor_computation} for details).
}
\label{fig:huffpost_stable_rank_neighbors}
\end{figure}

\begin{figure}[htbp]
\centering
\includegraphics[
    width=0.72\textwidth,
    trim=0.05cm 0.05cm 0.05cm 0.05cm,
    clip
]{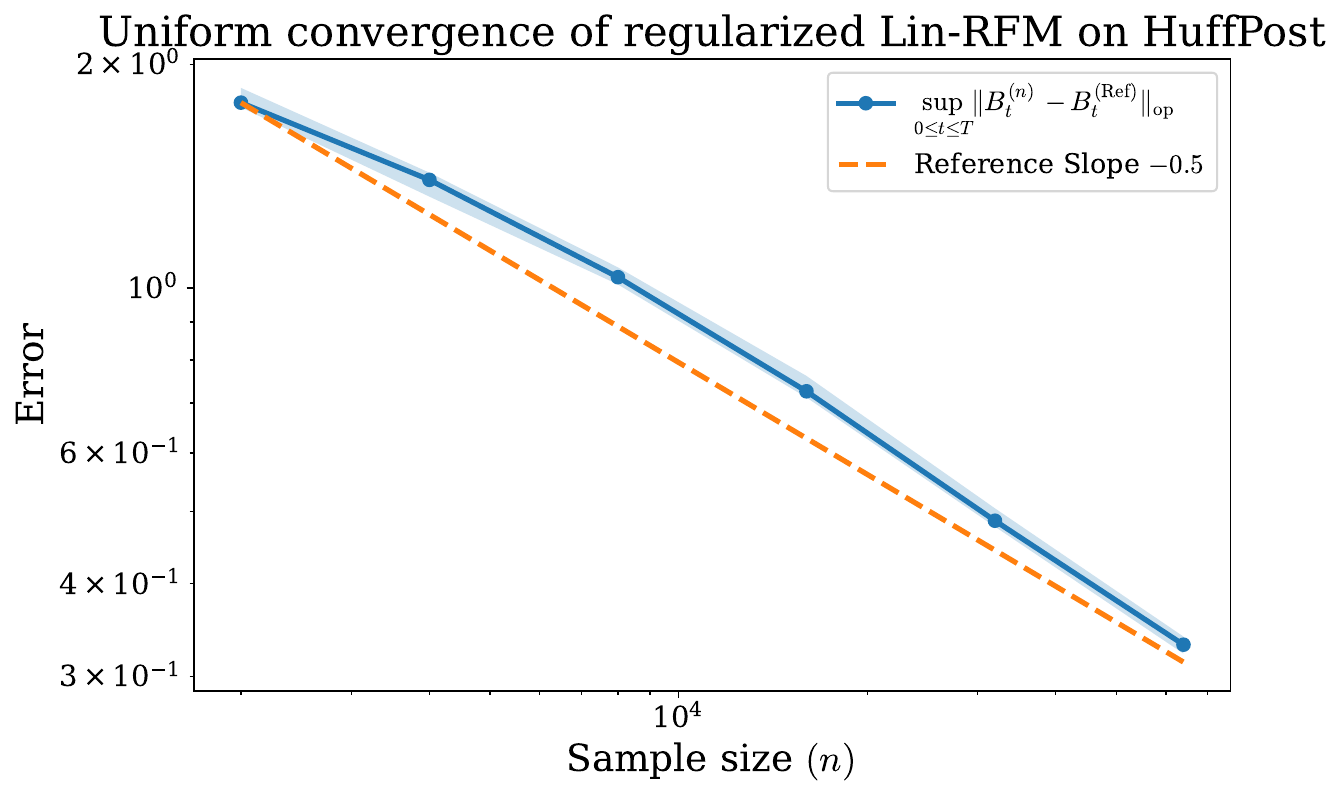}

\vspace{-0.35em}
\caption{
\textbf{HuffPost News feature tracking.}
We compute a large-sample reference trajectory \(B_t^{(\mathrm{Ref})}\) using the full HuffPost News dataset $(n=209,527)$. For each sample size \(n\), we draw subsamples without replacement, run regularized Lin-RFM for \(T=10\) iterations, and report
\(\sup_{0\le t\le T}\|B_t^{(n)}-B_t^{(\mathrm{Ref})}\|_{\mathrm{op}}\).
The empirical log-log slope is \(-0.493\), close to the reference \(n^{-1/2}\) rate, even though the HuffPost design is sparse, anisotropic, and outside the assumptions of Theorem~\ref{thm:deterministic-tracking}.
}
\label{fig:huffpost_feature_tracking}
\end{figure}

\FloatBarrier
\section{Complete nearest-neighbor tables}\label{app:neighbor-tables}
\subsection{Nearest neighbor embeddings}
\subsubsection{HLCA}
\begin{longtable}{@{}>{\raggedright\arraybackslash}p{0.12\textwidth}
                  >{\raggedright\arraybackslash}p{0.41\textwidth}
                  >{\raggedright\arraybackslash}p{0.41\textwidth}@{}}
\caption{\textbf{HLCA TF-IDF leave-one-out RFM gene-neighbor evolution.}
For each query gene, we report the top-5 nearest neighbors at \(t=1\) and \(t=5\), with cosine similarities in parentheses. Neighbors are computed from the \(B_t\)-induced gene embeddings with embedding dimension \(k=256\). Details in Appendix~\ref{app:hlca}.}
\label{tab:hlca_gene_neighbors_t1_t5_cosine}\\

\toprule
\textbf{Query} & \textbf{\(t=1\)} & \textbf{\(t=5\)} \\
\midrule
\endfirsthead

\toprule
\textbf{Query} & \textbf{\(t=1\)} & \textbf{\(t=5\)} \\
\midrule
\endhead

\bottomrule
\endfoot

\texttt{HBA1}
&
\makecell[l]{
\texttt{HBA2} (0.837)\\
\texttt{HBQ1} (0.822)\\
\texttt{TRDC} (0.816)\\
\texttt{LINC01505} (0.815)\\
\texttt{LIM2} (0.807)
}
&
\makecell[l]{
\texttt{LINC02446} (0.960)\\
\texttt{RGL4} (0.959)\\
\texttt{TESPA1} (0.956)\\
\texttt{ZNF683} (0.956)\\
\texttt{KLRC2} (0.953)
}
\\
\midrule

\texttt{CFTR}
&
\makecell[l]{
\texttt{TMEM92} (0.829)\\
\texttt{C16orf89} (0.804)\\
\texttt{SLC4A4} (0.793)\\
\texttt{PON3} (0.785)\\
\texttt{SHISA2} (0.771)
}
&
\makecell[l]{
\texttt{FREM2} (0.968)\\
\texttt{PDP2} (0.964)\\
\texttt{SLC6A20} (0.957)\\
\texttt{FOXP4} (0.957)\\
\texttt{RAB40C} (0.952)
}
\\
\midrule

\texttt{AVIL}
&
\makecell[l]{
\texttt{GARIN1A} (0.854)\\
\texttt{LINC02021} (0.831)\\
\texttt{ALG1L2} (0.819)\\
\texttt{CNN3-DT} (0.791)\\
\texttt{DISP2} (0.791)
}
&
\makecell[l]{
\texttt{EBF4} (0.929)\\
\texttt{TMEM8B} (0.919)\\
\texttt{FRAS1} (0.911)\\
\texttt{DLG2} (0.903)\\
\texttt{LINC01473} (0.899)
}
\\
\midrule

\texttt{FOXJ1}
&
\makecell[l]{
\texttt{LRRC71} (0.997)\\
\texttt{SPAG8} (0.995)\\
\texttt{DCDC2B} (0.994)\\
\texttt{PRR29} (0.994)\\
\texttt{GPR162} (0.994)
}
&
\makecell[l]{
\texttt{FAM216B} (0.995)\\
\texttt{MS4A8} (0.995)\\
\texttt{LDLRAD1} (0.994)\\
\texttt{CFAP43} (0.994)\\
\texttt{RSPH9} (0.993)
}
\\
\midrule

\texttt{CHGA}
&
\makecell[l]{
\texttt{ASCL1} (0.998)\\
\texttt{SCG3} (0.996)\\
\texttt{SST} (0.993)\\
\texttt{CALCA} (0.992)\\
\texttt{KIF1A} (0.990)
}
&
\makecell[l]{
\texttt{GFRA3} (0.996)\\
\texttt{SCG5} (0.996)\\
\texttt{CHGB} (0.996)\\
\texttt{ASCL1} (0.995)\\
\texttt{SCG3} (0.995)
}
\\
\midrule

\texttt{SFTPA1}
&
\makecell[l]{
\texttt{LAMP3} (0.697)\\
\texttt{ABCA3} (0.642)\\
\texttt{NAPSA} (0.640)\\
\texttt{SFTPD} (0.624)\\
\texttt{PEBP4} (0.591)
}
&
\makecell[l]{
\texttt{SFTPA2} (0.989)\\
\texttt{NAPSA} (0.987)\\
\texttt{SFTA2} (0.983)\\
\texttt{SFTPB} (0.982)\\
\texttt{LPCAT1} (0.980)
}
\\
\midrule


\texttt{KRT5}
&
\makecell[l]{
\texttt{S100A2} (0.987)\\
\texttt{PKP1} (0.960)\\
\texttt{TNS4} (0.954)\\
\texttt{CDH3} (0.952)\\
\texttt{PTPRZ1} (0.948)
}
&
\makecell[l]{
\texttt{PTPRZ1} (0.968)\\
\texttt{SERPINB5} (0.957)\\
\texttt{FAF1} (0.952)\\
\texttt{ARL4D} (0.950)\\
\texttt{TPBG} (0.945)
}
\\
\midrule

\texttt{MARCO}
&
\makecell[l]{
\texttt{MCEMP1} (0.976)\\
\texttt{CYP27A1} (0.964)\\
\texttt{ACP5} (0.963)\\
\texttt{VSIG4} (0.945)\\
\texttt{TREM1} (0.945)
}
&
\makecell[l]{
\texttt{VSIG4} (0.993)\\
\texttt{MS4A4A} (0.991)\\
\texttt{MCEMP1} (0.991)\\
\texttt{GLDN} (0.989)\\
\texttt{C1QC} (0.988)
}
\\
\midrule

\texttt{PECAM1}
&
\makecell[l]{
\texttt{RAMP3} (0.823)\\
\texttt{CLEC14A} (0.821)\\
\texttt{EGFL7} (0.818)\\
\texttt{GIMAP8} (0.807)\\
\texttt{LDB2} (0.775)
}
&
\makecell[l]{
\texttt{ENG} (0.962)\\
\texttt{RIPOR1} (0.940)\\
\texttt{ICAM2} (0.931)\\
\texttt{PRKCE} (0.931)\\
\texttt{STARD3} (0.924)
}
\\
\midrule

\texttt{RPL10A}
&
\makecell[l]{
\texttt{LINC00222} (0.144)\\
\texttt{LINC01179} (0.144)\\
\texttt{CENPN-AS1} (0.130)\\
\texttt{PSG2} (0.127)\\
\texttt{LINC02561} (0.123)
}
&
\makecell[l]{
\texttt{RPL3} (0.960)\\
\texttt{RPL7} (0.959)\\
\texttt{RPS17} (0.958)\\
\texttt{RPS18} (0.957)\\
\texttt{RPS9} (0.956)
}
\\
\midrule

\texttt{MT-ND1}
&
\makecell[l]{
\texttt{SNTN} (0.245)\\
\texttt{CFAP276} (0.240)\\
\texttt{RSPH1} (0.240)\\
\texttt{HOATZ} (0.239)\\
\texttt{CIMAP3} (0.238)
}
&
\makecell[l]{
\texttt{MT-ND4} (0.956)\\
\texttt{MT-ATP6} (0.955)\\
\texttt{MT-CYB} (0.954)\\
\texttt{MT-ND2} (0.952)\\
\texttt{MT-CO3} (0.952)
}
\\
\midrule

\texttt{TPI1}
&
\makecell[l]{
\texttt{ATP6V1G1} (0.908)\\
\texttt{PSME1} (0.883)\\
\texttt{PSMB1} (0.881)\\
\texttt{SKP1} (0.877)\\
\texttt{WDR83OS} (0.874)
}
&
\makecell[l]{
\texttt{RPLP0} (0.920)\\
\texttt{BTF3} (0.918)\\
\texttt{GAPDH} (0.917)\\
\texttt{RPL7A} (0.908)\\
\texttt{COX4I1} (0.908)
}
\\

\end{longtable}

\subsubsection{OpenWebText}
%

\scriptsize
\setlength{\tabcolsep}{2pt}
\renewcommand{\arraystretch}{1.08}
\providecommand{\categoryrule}{\specialrule{1.05pt}{3pt}{2pt}}
\providecommand{\queryrule}{\cmidrule(lr){1-3}}

\begingroup
\scriptsize
\setlength{\tabcolsep}{4pt}
\renewcommand{\arraystretch}{1.12}
\begin{longtable}{@{}p{2.0cm}p{5.85cm}p{5.85cm}@{}}
\caption{\textbf{Top-50 OpenWebText nearest-neighbor changes with cosine similarities.} Shown are the top-5 neighbors at \(t=1\) and \(t=5\) for the top-ranked query words. We picked this list based off the top-50 Jaccard disimilarity scores between the $t=1$ and $t=5$ nearest neighbor sets.}\label{tab:owt_top50_neighbors_cosine}\\
\toprule
\textbf{Query} & \textbf{\(\mathbf{t=1}\)} & \textbf{\(\mathbf{t=5}\)} \\
\midrule
\endfirsthead
\toprule
\textbf{Query} & \textbf{\(\mathbf{t=1}\)} & \textbf{\(\mathbf{t=5}\)} \\
\midrule
\endhead
\bottomrule
\endfoot

\textbf{\texttt{crushed}}
& \makecell[l]{lightly (0.796) \\toss (0.780) \\dip (0.777) \\tin (0.776) \\toast (0.773)}
& \makecell[l]{chopped (0.731) \\lemon (0.725) \\powder (0.709) \\dried (0.707) \\flour (0.699)}
\\
\midrule

\textbf{\texttt{buys}}
& \makecell[l]{gamble (0.822) \\buck (0.778) \\needless (0.761) \\theirs (0.759) \\geterrorstate (0.753)}
& \makecell[l]{sells (0.685) \\buying (0.662) \\buyer (0.650) \\seller (0.616) \\sell (0.602)}
\\
\midrule

\textbf{\texttt{identifies}}
& \makecell[l]{relates (0.786) \\acknowledges (0.780) \\recognizes (0.773) \\emphasize (0.763) \\illustrates (0.757)}
& \makecell[l]{identities (0.586) \\stereotypes (0.567) \\unidentified (0.567) \\identifying (0.565) \\identity (0.547)}
\\
\midrule

\textbf{\texttt{bounds}}
& \makecell[l]{alignment (0.773) \\secondly (0.770) \\needless (0.765) \\settling (0.764) \\illustrates (0.753)}
& \makecell[l]{ball (0.592) \\foul (0.584) \\defenders (0.573) \\basket (0.562) \\turnover (0.559)}
\\
\midrule

\textbf{\texttt{innocence}}
& \makecell[l]{awaiting (0.796) \\prosecuted (0.765) \\falsely (0.761) \\kidnapping (0.751) \\courtroom (0.750)}
& \makecell[l]{conviction (0.751) \\prosecution (0.738) \\jury (0.735) \\defendant (0.725) \\convictions (0.692)}
\\
\midrule

\textbf{\texttt{rouge}}
& \makecell[l]{needless (0.765) \\raymond (0.758) \\geterrorstate (0.758) \\apiculture (0.756) \\glen (0.756)}
& \makecell[l]{suburb (0.497) \\unarmed (0.464) \\qu (0.463) \\midwest (0.455) \\rapids (0.438)}
\\
\midrule

\textbf{\texttt{marking}}
& \makecell[l]{destined (0.755) \\hailed (0.755) \\conjunction (0.755) \\unofficial (0.754) \\anticipation (0.753)}
& \makecell[l]{marked (0.589) \\marks (0.557) \\celebrating (0.545) \\celebrated (0.529) \\milestone (0.501)}
\\
\midrule

\textbf{\texttt{improves}}
& \makecell[l]{drastically (0.754) \\negatively (0.753) \\vastly (0.752) \\diminished (0.748) \\noticeable (0.745)}
& \makecell[l]{improving (0.672) \\improved (0.666) \\enhance (0.654) \\improve (0.641) \\enhanced (0.604)}
\\
\midrule

\textbf{\texttt{hyper}}
& \makecell[l]{geterrorstate (0.766) \\apiculture (0.765) \\relentless (0.739) \\transforming (0.738) \\myriad (0.733)}
& \makecell[l]{server (0.544) \\vm (0.540) \\sql (0.531) \\socket (0.518) \\servers (0.516)}
\\
\midrule

\textbf{\texttt{classics}}
& \makecell[l]{delight (0.755) \\glorious (0.749) \\fashioned (0.746) \\aforementioned (0.744) \\beautifully (0.743)}
& \makecell[l]{classic (0.560) \\genre (0.531) \\favourite (0.500) \\vintage (0.497) \\indie (0.488)}
\\
\midrule

\textbf{\texttt{disposal}}
& \makecell[l]{cement (0.757) \\machinery (0.743) \\deter (0.741) \\feasible (0.737) \\adequately (0.732)}
& \makecell[l]{radioactive (0.653) \\waste (0.590) \\contamination (0.575) \\reactor (0.545) \\reactors (0.543)}
\\
\midrule

\textbf{\texttt{refined}}
& \makecell[l]{wax (0.772) \\combines (0.741) \\incorporate (0.735) \\simplicity (0.728) \\reminiscent (0.725)}
& \makecell[l]{vegetable (0.616) \\olive (0.592) \\coconut (0.577) \\petroleum (0.555) \\ingredient (0.550)}
\\
\midrule

\textbf{\texttt{catalog}}
& \makecell[l]{aforementioned (0.748) \\lastly (0.746) \\needless (0.734) \\conjunction (0.730) \\entirety (0.730)}
& \makecell[l]{vinyl (0.603) \\retailer (0.561) \\ebay (0.531) \\merchandise (0.528) \\editions (0.501)}
\\
\midrule

\textbf{\texttt{harmony}}
& \makecell[l]{theirs (0.746) \\contradiction (0.740) \\visions (0.735) \\embracing (0.730) \\transforming (0.727)}
& \makecell[l]{essence (0.579) \\virtue (0.555) \\buddha (0.544) \\desires (0.542) \\ideals (0.528)}
\\
\midrule

\textbf{\texttt{aurora}}
& \makecell[l]{shannon (0.741) \\raymond (0.739) \\geterrorstate (0.729) \\apiculture (0.728) \\rouge (0.726)}
& \makecell[l]{shootings (0.478) \\skies (0.462) \\shoots (0.456) \\gunman (0.447) \\shooting (0.440)}
\\
\midrule

\textbf{\texttt{nova}}
& \makecell[l]{rcmp (0.749) \\calgary (0.738) \\winnipeg (0.724) \\provinces (0.720) \\edmonton (0.719)}
& \makecell[l]{ontario (0.809) \\ottawa (0.788) \\canada (0.777) \\quebec (0.777) \\canadians (0.777)}
\\
\midrule

\textbf{\texttt{mockup}}
& \makecell[l]{lastly (0.735) \\beacon (0.735) \\needless (0.733) \\tc (0.723) \\fischer (0.717)}
& \makecell[l]{download (0.727) \\pdf (0.632) \\downloads (0.585) \\gluten (0.552) \\downloaded (0.541)}
\\
\midrule

\textbf{\texttt{reckless}}
& \makecell[l]{irresponsible (0.751) \\intentional (0.732) \\falsely (0.724) \\distracted (0.716) \\fined (0.715)}
& \makecell[l]{unlawful (0.614) \\felony (0.605) \\kidnapping (0.589) \\committing (0.587) \\charged (0.560)}
\\
\midrule

\textbf{\texttt{smallest}}
& \makecell[l]{chunk (0.732) \\lengths (0.731) \\scales (0.729) \\sheer (0.727) \\measuring (0.719)}
& \makecell[l]{sizes (0.746) \\size (0.730) \\smaller (0.713) \\larger (0.672) \\sized (0.637)}
\\
\midrule

\textbf{\texttt{proximity}}
& \makecell[l]{confined (0.740) \\stretches (0.739) \\settling (0.722) \\desirable (0.717) \\aforementioned (0.715)}
& \makecell[l]{populated (0.557) \\surrounding (0.504) \\distances (0.499) \\farther (0.485) \\close (0.485)}
\\
\midrule

\textbf{\texttt{discoveries}}
& \makecell[l]{discovering (0.766) \\myths (0.724) \\phenomena (0.721) \\darwin (0.716) \\traces (0.706)}
& \makecell[l]{discovery (0.736) \\scientific (0.690) \\scientists (0.672) \\biology (0.667) \\evolution (0.637)}
\\
\midrule

\textbf{\texttt{horn}}
& \makecell[l]{resemble (0.731) \\nod (0.728) \\reminiscent (0.722) \\proceeded (0.722) \\moose (0.716)}
& \makecell[l]{guinea (0.556) \\africa (0.547) \\elephants (0.504) \\americas (0.501) \\elephant (0.498)}
\\
\midrule

\textbf{\texttt{surfaced}}
& \makecell[l]{confirming (0.738) \\accompanying (0.728) \\rumor (0.724) \\verge (0.717) \\captures (0.711)}
& \makecell[l]{uploaded (0.612) \\clips (0.599) \\viral (0.597) \\footage (0.577) \\disturbing (0.574)}
\\
\midrule

\textbf{\texttt{brent}}
& \makecell[l]{shawn (0.721) \\casey (0.719) \\dale (0.713) \\curtis (0.712) \\shannon (0.712)}
& \makecell[l]{crude (0.627) \\rig (0.560) \\oil (0.550) \\sands (0.543) \\spill (0.534)}
\\
\midrule

\textbf{\texttt{guarantees}}
& \makecell[l]{agreeing (0.719) \\ensures (0.719) \\granting (0.714) \\mandated (0.710) \\securing (0.705)}
& \makecell[l]{guarantee (0.742) \\obligations (0.600) \\guaranteed (0.571) \\obligation (0.544) \\protects (0.512)}
\\
\midrule

\textbf{\texttt{hatch}}
& \makecell[l]{gardner (0.727) \\claire (0.726) \\blasted (0.713) \\akin (0.704) \\blunt (0.697)}
& \makecell[l]{senate (0.603) \\schumer (0.577) \\mcconnell (0.577) \\senators (0.577) \\mitch (0.568)}
\\
\midrule

\textbf{\texttt{reservation}}
& \makecell[l]{lodge (0.749) \\settling (0.715) \\geterrorstate (0.702) \\apiculture (0.701) \\moose (0.699)}
& \makecell[l]{tribe (0.639) \\tribes (0.631) \\indians (0.587) \\tribal (0.587) \\pradesh (0.536)}
\\
\midrule

\textbf{\texttt{logistics}}
& \makecell[l]{bolster (0.739) \\organizational (0.708) \\advise (0.706) \\tasked (0.706) \\conjunction (0.701)}
& \makecell[l]{operations (0.637) \\operational (0.619) \\marines (0.555) \\deployed (0.553) \\personnel (0.532)}
\\
\midrule

\textbf{\texttt{cleaner}}
& \makecell[l]{cement (0.727) \\converting (0.726) \\footprint (0.715) \\cleaned (0.696) \\recycling (0.694)}
& \makecell[l]{fuels (0.688) \\clean (0.677) \\coal (0.626) \\fossil (0.618) \\pollution (0.611)}
\\
\midrule

\textbf{\texttt{poles}}
& \makecell[l]{observing (0.729) \\unusually (0.724) \\hollow (0.713) \\circular (0.704) \\dug (0.684)}
& \makecell[l]{clouds (0.617) \\pole (0.613) \\comet (0.608) \\venus (0.593) \\jupiter (0.583)}
\\
\midrule

\textbf{\texttt{influx}}
& \makecell[l]{fleeing (0.740) \\serbia (0.719) \\migrant (0.690) \\seekers (0.689) \\hungarian (0.688)}
& \makecell[l]{migrants (0.788) \\migration (0.715) \\refugees (0.714) \\asylum (0.701) \\refugee (0.698)}
\\
\midrule

\textbf{\texttt{ranch}}
& \makecell[l]{backyard (0.724) \\glen (0.723) \\lodge (0.705) \\greene (0.686) \\pine (0.682)}
& \makecell[l]{hills (0.582) \\springs (0.574) \\acre (0.557) \\creek (0.551) \\acres (0.537)}
\\
\midrule

\textbf{\texttt{comeback}}
& \makecell[l]{eventual (0.729) \\deserved (0.726) \\triumph (0.711) \\sweep (0.677) \\crushing (0.663)}
& \makecell[l]{victory (0.656) \\win (0.602) \\beat (0.601) \\victories (0.591) \\winning (0.586)}
\\
\midrule

\textbf{\texttt{arch}}
& \makecell[l]{carved (0.713) \\geterrorstate (0.699) \\apiculture (0.699) \\lodge (0.699) \\hollow (0.691)}
& \makecell[l]{kernel (0.523) \\gnome (0.511) \\ubuntu (0.505) \\linux (0.491) \\debian (0.487)}
\\
\midrule

\textbf{\texttt{splash}}
& \makecell[l]{lastly (0.732) \\sneak (0.708) \\popping (0.700) \\stunned (0.677) \\shields (0.669)}
& \makecell[l]{melee (0.633) \\damage (0.588) \\shield (0.555) \\bonus (0.547) \\dealt (0.546)}
\\
\midrule

\textbf{\texttt{mann}}
& \makecell[l]{klein (0.713) \\raymond (0.702) \\fischer (0.691) \\neill (0.686) \\kenneth (0.684)}
& \makecell[l]{mit (0.591) \\von (0.585) \\dem (0.582) \\ist (0.581) \\das (0.576)}
\\
\midrule

{\tiny\textbf{\texttt{constituents}}}
& \makecell[l]{contentious (0.701) \\unpopular (0.698) \\irresponsible (0.689) \\outraged (0.689) \\outspoken (0.677)}
& \makecell[l]{legislators (0.538) \\lawmakers (0.531) \\legislative (0.522) \\bipartisan (0.499) \\caucus (0.495)}
\\
\midrule

\textbf{\texttt{feeds}}
& \makecell[l]{accustomed (0.701) \\myriad (0.697) \\continually (0.688) \\needless (0.681) \\lastly (0.679)}
& \makecell[l]{feed (0.636) \\streams (0.459) \\rss (0.449) \\profiles (0.427) \\facebook (0.426)}
\\
\midrule

\textbf{\texttt{lesser}}
& \makecell[l]{likewise (0.706) \\absent (0.688) \\aforementioned (0.684) \\consequently (0.681) \\inclined (0.661)}
& \makecell[l]{convictions (0.537) \\conviction (0.513) \\kidnapping (0.513) \\counts (0.510) \\convicted (0.510)}
\\
\midrule

\textbf{\texttt{dock}}
& \makecell[l]{departed (0.701) \\pins (0.690) \\lastly (0.680) \\aforementioned (0.675) \\specially (0.674)}
& \makecell[l]{cargo (0.553) \\vessel (0.549) \\ship (0.544) \\aboard (0.538) \\ports (0.535)}
\\
\midrule

\textbf{\texttt{monica}}
& \makecell[l]{lauren (0.695) \\glen (0.685) \\carol (0.682) \\richards (0.681) \\melissa (0.675)}
& \makecell[l]{angeles (0.549) \\barbara (0.527) \\los (0.524) \\sacramento (0.494) \\lakers (0.466)}
\\
\midrule

\textbf{\texttt{freddie}}
& \makecell[l]{troubled (0.695) \\timothy (0.682) \\crushing (0.676) \\kurt (0.672) \\curtis (0.669)}
& \makecell[l]{mortgage (0.593) \\securities (0.562) \\homeowners (0.528) \\bonds (0.528) \\lending (0.517)}
\\
\midrule

\textbf{\texttt{segregation}}
& \makecell[l]{exclusion (0.713) \\discriminatory (0.672) \\predominantly (0.666) \\advocated (0.664) \\injustice (0.664)}
& \makecell[l]{racial (0.703) \\blacks (0.647) \\whites (0.543) \\slavery (0.537) \\racism (0.518)}
\\
\midrule

\textbf{\texttt{bust}}
& \makecell[l]{gamble (0.680) \\crushing (0.669) \\needle (0.668) \\popping (0.666) \\hooked (0.665)}
& \makecell[l]{illicit (0.531) \\heroin (0.511) \\cocaine (0.490) \\dealers (0.490) \\crack (0.489)}
\\
\midrule

\textbf{\texttt{dependence}}
& \makecell[l]{ineffective (0.688) \\potent (0.685) \\contributes (0.655) \\diminished (0.652) \\extraction (0.644)}
& \makecell[l]{addiction (0.613) \\opioid (0.594) \\substance (0.572) \\heroin (0.558) \\illicit (0.557)}
\\
\midrule

\textbf{\texttt{franchises}}
& \makecell[l]{aforementioned (0.667) \\continuity (0.665) \\ninja (0.661) \\wildly (0.658) \\saga (0.658)}
& \makecell[l]{franchise (0.681) \\legends (0.609) \\leagues (0.539) \\mlb (0.517) \\ea (0.515)}
\\
\midrule

\textbf{\texttt{arranged}}
& \makecell[l]{instructed (0.670) \\proceeded (0.667) \\consisted (0.659) \\consisting (0.651) \\lodge (0.649)}
& \makecell[l]{wedding (0.545) \\arrangement (0.533) \\separated (0.524) \\divorced (0.515) \\marry (0.502)}
\\
\midrule

\textbf{\texttt{inputs}}
& \makecell[l]{switches (0.663) \\analog (0.658) \\assign (0.657) \\optimal (0.656) \\diagram (0.654)}
& \makecell[l]{input (0.698) \\parameters (0.603) \\output (0.561) \\linear (0.528) \\variables (0.522)}
\\
\midrule

\textbf{\texttt{boulder}}
& \makecell[l]{greene (0.691) \\aurora (0.661) \\jacobs (0.645) \\davidson (0.644) \\rapids (0.640)}
& \makecell[l]{montgomery (0.564) \\madison (0.529) \\county (0.518) \\springs (0.508) \\delaware (0.490)}
\\
\midrule

\textbf{\texttt{ramp}}
& \makecell[l]{pushes (0.661) \\bump (0.658) \\cement (0.655) \\crashing (0.646) \\sidewalk (0.644)}
& \makecell[l]{pedestrian (0.565) \\lane (0.549) \\lanes (0.531) \\intersection (0.527) \\boulevard (0.518)}
\\
\end{longtable}
\endgroup

\end{document}